\documentclass{article}
\usepackage{svg}

\usepackage{microtype}
\usepackage{graphicx}
\usepackage{subfigure}
\usepackage{booktabs} 

\usepackage{hyperref}

\usepackage[accepted]{icml2025}

\usepackage{mymacro}
\usepackage{lipsum}

\usepackage[capitalize,noabbrev]{cleveref}

\theoremstyle{plain}
\newtheorem{theorem}{Theorem}[section]
\newtheorem{proposition}[theorem]{Proposition}
\newtheorem{lemma}[theorem]{Lemma}

\theoremstyle{definition}
\newtheorem{definition}[theorem]{Definition}
\newtheorem{assumption}[theorem]{Assumption}
\theoremstyle{remark}
\newtheorem{remark}[theorem]{Remark}

\usepackage[textsize=tiny]{todonotes}

\icmltitlerunning{Accelerating Differentially Private Federated Learning via Adaptive Extrapolation}

\begin{document}

\twocolumn[
    \icmltitle{Accelerating Differentially Private Federated Learning \\ via Adaptive Extrapolation}



    \icmlsetsymbol{equal}{*}

    \begin{icmlauthorlist}
        \icmlauthor{Shokichi Takakura}{comp}
        \icmlauthor{Seng Pei Liew}{comp}
        \icmlauthor{Satoshi Hasegawa}{comp}
    \end{icmlauthorlist}

    \icmlaffiliation{comp}{LY Corporation, Tokyo, Japan}

    \icmlcorrespondingauthor{Shokichi Takakura}{stakakur@lycorp.co.jp}

    \icmlkeywords{Machine Learning, ICML}

    \vskip 0.3in
]



\printAffiliationsAndNotice{} 

\begin{abstract}
    The federated learning (FL) framework enables multiple clients to collaboratively train machine learning models without sharing their raw data,
    but it remains vulnerable to privacy attacks.
    One promising approach is to incorporate differential privacy (DP)—a formal notion of privacy—into the FL framework.
    DP-FedAvg is one of the most popular algorithms for DP-FL,
    but it is known to suffer from the slow convergence in the presence of heterogeneity among clients' data.
    Most of the existing methods to accelerate DP-FL require 1) additional hyperparameters or 2) additional computational cost for clients,
    which is not desirable since 1) hyperparameter tuning is computationally expensive and data-dependent choice of hyperparameters raises the risk of privacy leakage,
    and 2) clients are often resource-constrained.
    To address this issue, we propose DP-FedEXP, which adaptively selects the global step size based on the diversity of the local updates
    without requiring any additional hyperparameters or client computational cost.
    We show that DP-FedEXP provably accelerates the convergence of DP-FedAvg and
    it empirically outperforms existing methods tailored for DP-FL.
\end{abstract}

\section{Introduction}
Federated learning (FL)~\citep{konečný2017federated} is a distributed machine learning framework
where multiple clients collaboratively train a global model without sharing their raw data.
FL has been widely adopted in various applications, such as mobile devices, edge devices, and healthcare systems
where data is sensitive and cannot be shared due to privacy concerns~\citep{kairouz2021advances,xu2023federated}.
Due to its simplicity, stateless property, and communication efficiency,
FedAvg~\citep{mcmahan2017communication} is one of the most popular FL algorithms.
In FedAvg, the server sends the global model to the clients, and each client performs a multiple-step local training using stochastic gradient descent (SGD)
to reduce the communication cost. Then, the clients send the local updates to the server and the server aggregates the updates by averaging them.
Although FL algorithms are intended to protect the privacy of clients,
several works have shown that there is a potential leakage of privacy from local updates~\citep{lam2021gradient,geiping2020inverting,nasr2019comprehensive,zhao2024federated}.
For example,~\citet{lam2021gradient} has shown that an attacker can recover privileged information from aggregated model updates in FL.
Taking into account the growing concern for privacy in the field of machine learning,
incorporating formal privacy guarantees into FL is a crucial and fundamental challenge.

One promising approach to tackle the privacy issue in FL is to add noise to the updates of the model to ensure differential privacy (DP)~\citep{dwork2006calibrating},
which is a general and mathematically rigorous notion to quantify the degree of privacy protection.
A practical approach to incorporate DP to the FL framework is DP-FedAvg~\citep{mcmahan2017learning}, which is a DP extension of FedAvg.
Unfortunately, (DP-)FedAvg has been known to suffer from slow convergence in the presence of data heterogeneity across clients.
This issue is known as \textit{ the client drift error}~\citep{karimireddy2019error}.
The effect of the client drift error becomes more severe when only a subset of all clients participate in each training round~\citep{kairouz2021advances}.

To deal with data heterogeneity, a line of work has studied variance reduction techniques in (non-private) FL setting~\citep{karimireddy2020mime,karimireddy2020scaffold,mitra2021linear}.
Extending the above techniques to the DP setting, DP-SCAFFOLD~\citep{noble2022differentially} has been proposed and shown to achieve improved convergence guarantee.
Although the above methods enjoy theoretically favorable properties, they require clients to be stateful and additional computational cost in clients.
This is impractical since clients are often resource-constrained.

Another line of work has sought to accelerate the convergence of (DP-)FedAvg by regarding the local updates as pseudo-gradients
and updating the global model using global optimization algorithms such as Adam~\citep{kingma2015adam} with additional hyperparameters such as global step size~\citep{reddi2021adaptive}.
Although the performance crucially relies on the choice of the hyperparameters,
it is difficult to obtain the optimal hyperparameters in the DP settings
since hyperparameter tuning on sensitive data leads to additional privacy leakage~\cite{papernot2021hyperparameter}.
Furthermore, it is highly costly in practice to tune the hyperparameters in the FL setting,
since the data is distributed across clients.

To develop an effective and practical DP-FL algorithm, we pose the following question:

\textit{Can DP-FL be accelerated under heterogeneity of client data without any additional hyperparameters and computational cost for clients?}

In this paper, to address the above question, we propose DP-FedEXP by incorporating FedEXP~\citep{jhunjhunwala2023fedexp},
which adaptively determines the global step size to the heterogeneity of the local updates, into the DP-FL framework in a non-trivial way.
Specifically, we consider the two different scenarios of DP: Local Differential Privacy (LDP) and Central Differential Privacy (CDP).
We found that the step size formula for FedEXP cannot be directly extended in both cases.
Thus, we carefully design the step size formula for LDP and CDP and develop a simple but effective framework to accelerate the convergence of existing DP-FL algorithms.
We would like to emphasize that our proposed method is \textit{orthogonal} to existing works
which try to accelerate (DP-)FL by modifying the local training procedure~\citep{li2020federated,karimireddy2020scaffold,noble2022differentially,shi2023make}
and thus, it can be combined with them to further improve the performance.

Our contribution can be summarized as follows:
\begin{itemize}
    \item We propose LDP-FedEXP and CDP-FedEXP with simple but effective parameter-free step size rules in DP-FL.
    \item We provide formal differential privacy guarantee and convergence guarantees for general non-convex objectives.
          We prove that the proposed method provably accelerates the convergence in the presence of data heterogeneity.
    \item In the numerical experiments, we show that DP-FedEXP outperforms existing algorithms in utility while preserving the privacy guarantee.
\end{itemize}



\subsection{Other Related Work}


\paragraph{Adapive Optimization Algorithms with DP}
Inspired by the success of adaptive optimization algorithms such as Adam~\citep{kingma2015adam} in the non-private setting,
their DP variants have been utilized in various fields~\citep{li2021large,daigavane2022node}.
However, despite their success in the non-private setting,
their DP variants often suffer from the slow convergence.
\citet{tang2024dp} have found that the bias from DP noise degrades the performance of DP-Adam
and proposed DP-AdamBC, which removes the bias in the second moment estimation of Adam update.
This implies that it is not straightforward to extend adaptive methods in the non-private setting to the DP setting.
Note that the above attempts are mainly focused on the centralized setting and
it is still unclear how to incorporate the adaptivity to the heterogeneity of the client data into DP-FL algorithms.

\paragraph{Hyperparameter Tuning with DP}
In the most of the existing works, the privacy leakage from hyperparameter tuning is ignored.
However, as discussed in~\citet{papernot2021hyperparameter}, hyperparameters can raise the privacy risks of memorizing the training data.
Several works~\citep{liu2019private,wang2023dp,papernot2021hyperparameter,mohapatra2022role} have proposed to privatize hyperparameter tuning by consuming additional privacy budget.
However, these methods often result in much weaker privacy guarantees unless larger DP noise is used.
For example,~\citet{papernot2021hyperparameter} have reported that the privacy parameter can be doubled or even tripled by accounting the privacy leakage from hyperparameter tuning.
Furthermore, it is prohibitively expensive or even infeasible
to conduct hyperparameter tuning with distributed data in the FL setting.
\vspace{-0.5cm}
\paragraph{Hyperparameter-Free DP Optimization}
A line of work has investigated adaptive methods to select hyperparameters for DP optimization algorithms~\citep{andrew2021differentially,bu2024automatic,anonymous2024towards}.
For example, Adaptive clipping~\citep{andrew2021differentially} selects clipping threshold in DP-FL by estimating a quantile of the update norm with a negligible amount of privacy budget.
Furthermore,~\citet{anonymous2024towards} have proposed a hyperparameter-free algorithm for DP optimization in the centralized setting.
However, to the best of our knowledge, there is no work that provides hyperparameter-free step size rule to deal with the heterogeneity of the client data for DP-FL.

\section{Problem Settings and Preliminaries}
In this section, we introduce the problem settings of federated learning~\citep{konečný2017federated} and the notion of differential privacy~\citep{dwork2006calibrating}.
We also review  previous works and the motivation of our proposed method.

\subsection{Federated Learning}
In this paper, we consider a federated learning setting where there are a central server and $M$ clients, which have their own local datasets with sensitive information.
The objective is to minimize the following empirical risk:
\begin{align}
    \min_{w \in \mathbb{R}^d} F(w) := \frac{1}{M} \sum_{i=1}^M f_i(w), \label{eq:problem}
\end{align}
where $w \in \R^d$ is the parameter of the model, $M$ is the number of clients and $f_i(w) := \frac{1}{\abs{\mathcal{D}_i}} \sum_{d_i \in \mathcal{D}_i} l(w, d_i)$ is the loss function of the $i$-th client computed on a loss function $l$ and the local dataset $\mathcal{D}_i$.

\subsection{Differential Privacy}
In this paper, we consider two scenarios of differential privacy: Central Differential Privacy (CDP) and Local Differential Privacy (LDP) .
In the CDP setting, we assume that the central server is trusted and provide the privacy guarantee to the attackers who can access only the updated model.
On the other hand, in the LDP setting, we do not assume any trusted server and provide the privacy guarantee to the attackers who can access the local updates.
Since LDP does not assume the trusted server, it is more challenging to achieve the privacy guarantee than in the CDP setting while maintaining the utility.
Here, we provide the formal definitions of $(\varepsilon, \delta)$-CDP and $(\varepsilon, \delta)$-LDP.

\begin{definition}[Central Differential Privacy~\cite{dwork2014algorithmic}]
    Let $\mathcal{X}$ be the set of all possible client datasets.
    A central randomized mechanism $\mathcal{Q}: \mathcal{X}^M \to \mathcal{Y}$ satisfies $(\varepsilon, \delta)$-CDP if
    for any two neighboring inputs $x, x' \in \mathcal{X}^M$, which differ in one client dataset,
    we have
    \begin{align*}
        \forall S \subset \mathcal{Y}: \Pr[\mathcal{Q}(x) \in S] \leq e^\varepsilon \Pr[\mathcal{Q}(x') \in S] + \delta.
    \end{align*}
\end{definition}
\begin{definition}[Local Differential Privacy~\cite{kasiviswanathan2011can}]
    Let $\mathcal{X}$ be the set of all possible client datasets.
    A local randomized mechanism $\mathcal{R}: \mathcal{X} \to \mathcal{Y}$ satisfies $(\varepsilon, \delta)$-LDP if for any two inputs $x, x' \in \mathcal{X}$,
    we have
    \begin{align*}
        \forall S \subset \mathcal{Y}: \Pr[\mathcal{R}(x) \in S] \leq e^\varepsilon \Pr[\mathcal{R}(x') \in S] + \delta.
    \end{align*}
    If $\delta = 0$, $\mathcal{R}$ is called to satisfy \textit{pure differential privacy}.
\end{definition}
In the above definitions, we employ \textit{client-level} DP, which protects whole dataset for each client.
This is a stronger notion of privacy compared to \textit{sample-level} DP, which protects each sample in clients' datasets.
Client-level DP is suitable for the FL setting with a large number of clients such as mobile devices and edge devices.
\subsection{DP-FedAvg}
DP-FedAvg~\citep{mcmahan2017learning} is one of the most popular algorithms for federated learning with differential privacy
due to its simplicity and communication efficiency.
At round $t$, the server sends the global model $w^{(t - 1)}$ to all clients.
Then, each client performs $\tau$ steps of local training $w_i^{(t - 1, 0)} := w^{(t - 1)}, w_i^{(t - 1, k)} := w_i^{(t - 1, k - 1)} - \eta_l \grad f_i(w_i^{(t - 1, k - 1)})~(k=1 \dots \tau)$ using (stochastic) gradient descent with step size $\eta_l$
as in Algorithm~\ref{alg:localupdate}
and computes the local update $\tilde \Delta_i^{(t)} := w_i^{(t - 1, \tau)} - w^{(t - 1)}$.
To bound the sensitivity of the local updates, each client $i$ applies clipping to their local update $\Delta_i^{(t)} := \min\{C / \|\tilde \Delta_i^{(t)}\|, 1\} \cdot \tilde \Delta_i^{(t)}$ with threshold $C > 0$.
Then, each client sends the central server the local update $\Delta_i^{(t)}$ in the CDP setting and the randomized update $c_i^{(t)} := \localrandomizer(\Delta^{(t)}_i)$ in the LDP setting.
The central server aggregates the local updates as follows:
\begin{align*}
    \begin{cases}
        \bar c^{(t)} & := \frac{1}{M} \sum_{i=1}^M c_i^{(t)} \quad (\text{LDP setting}),                          \\
        \bar c^{(t)} & := \frac{1}{M} \sum_{i=1}^M \Delta_i^{(t)} + \varepsilon^{(t)} \quad (\text{CDP setting}),
    \end{cases}
\end{align*}
where $\varepsilon^{(t)}$ follows Gaussian $\mathcal{N}(0, \sigma^2 / M)$.

A natural choice of $\localrandomizer$ is Gaussian mechanism, which adds Gaussian noise to the local updates as $c^{(t)}_i = \Delta^{(t)}_i + \varepsilon^{(t)}_i$ for $\varepsilon^{(t)}_i \sim \mathcal{N}(0, \sigma^2)$.
However, Gaussian mechanism does not satisfy pure differential privacy.
\textit{PrivUnit}~\citep{bhowmick2018protection} is known as a local randomizer which satisfies the pure differential privacy.
Moreover, PrivUnit achieves the asymptotically optimal trade-off between privacy and utility~\cite{bhowmick2018protection,asi2022optimal}.
In this paper, we consider both Gaussian mechanism and PrivUnit as a local randomizer in the LDP setting
and prove the privacy and convergence guarantees in Section~\ref{sec:utility}.

For PrivUnit, we follow the procedure in \citet{bhowmick2018protection} and
privatize the norm and the direction of the local update separately.
That is, we randomize the local update $\Delta_i^{(t)}$ as follows:
\begin{align*}
    c_i^{(t)} = \hat r_i^{(t)} \cdot z_i^{(t)},
\end{align*}
where $z_i^{(t)} := \privunit\ab(\Delta_i^{(t)} / \|\Delta_i^{(t)}\|;\varepsilon_0, \varepsilon_1)$, $\hat r_i^{(t)} :=  \scalardp\ab(\|\Delta_i^{(t)}\|; \varepsilon_2)$,
and $\varepsilon_0, \varepsilon_1, \varepsilon_2$ are privacy parameters.
Here, $\privunit$ privatizes the direction and $\scalardp$ privatizes the norm. See Algorithm~\ref{alg:privunit} and~\ref{alg:scalardp} for the detailed procedure.
As shown in~\citet{bhowmick2018protection}, $c_i^{(t)}$ is an unbiased estimator of $\Delta_i^{(t)}$ and its variance is bounded by $O(dC^2 \cdot (\frac{1}{\varepsilon_1} \vee \frac{1}{(e^{\varepsilon_1} - 1)^2}))$ if $\varepsilon_1 \in (0, d)$ and $\varepsilon_2 = \Omega(1)$.
We define $\sigma^2 := C^2 \cdot (\frac{1}{\varepsilon_1} \vee \frac{1}{(e^{\varepsilon_1} - 1)^2})$ for the PrivUnit case to ensure the consistency in the notation with the Gaussian mechanism case, where the variance of $c_i^{(t)}$ is given by $d\sigma^2$.

In DP-FedAvg, the server updates the global model by just adding the averaged local update
as $w^{(t + 1)} = w^{(t)} + \bar c^{(t)}$.
To accelerate the convergence, several works deal with the noisy local updates as the pseudo-gradients and update the global model using the global learning rate~\citep{reddi2021adaptive,noble2022differentially}.
That is, the global model is updated as
\begin{align*}
    w^{(t + 1)} = w^{(t)} + \eta_g \bar c^{(t)},
\end{align*}
where $\eta_g$ is a global step size. Note that $\eta_g = 1$ recovers DP-FedAvg.
To ensure the convergence, $\eta_g$ should be chosen carefully.
Previous works have discussed the optimal global step size~\citep{zhang2022understanding} but
it is difficult in practice to tune such a hyperparameter with formal DP guarantee since hyperparameter tuning is computationally expensive and requires additional privacy budget~\citep{papernot2021hyperparameter}.
To fill the gap between the theory and practice, it is desirable to determine the step size \textit{in an adaptive manner}.

\subsection{FedEXP} \label{sec:proposed-method}
In the context of non-DP federated learning, FedEXP~\citep{jhunjhunwala2023fedexp} and FedEXProx~\citep{li2024the} have been proposed to determine the global step size adaptively to the heterogeneity of the local updates.
Their key idea is the analogy between FedAvg and Projection Onto Convex Sets (POCS) algorithm in the overparameterized convex regime.
Following the adaptive step size rule of POCS~\citep{pierra2011decomposition}, they define the global step size as
\begin{align}
    \eta_g^{(t)} := \frac{\frac{1}{M}\sum_{i=1}^M \norm{\Delta_i^{(t)}}^2}{\norm{\bar \Delta^{(t)}}^2}, \label{eq:step-size-fedexp}
\end{align}
where $\bar \Delta^{(t)} = \frac{1}{M} \sum_{i=1}^M \Delta_i^{(t)}$ is the average of the local updates.
Here, we follow the formula in~\citet{li2024the} and omit the coefficient $1/2$ and a small constant added to the denominator, which appear in~\citet{jhunjhunwala2023fedexp} since the convergence analysis in~\citet{jhunjhunwala2023fedexp} does not require these factors.
In the case of $\tau = 1$, the above formula is reduced to $\frac{\frac{1}{M}\sum_{i=1}^M \norm{\grad f_i(w^{(t)})}^2}{\norm{\grad F(w^{(t)})}^2}$,
which is known as a measure of the heterogeneity among clients~\citep{haddadpour2019convergence,wang2020tackling}.
Thus, FedEXP adaptively determines the global step size based on the diversity of the clients' data.
Although FedEXP has been shown to accelerate the convergence in the non-private setting,
it is still unclear how to extend the algorithm to the DP setting.

\section{Proposed Method: DP-FedEXP}

In this section, we propose DP-FedEXP (LDP-FedEXP and CDP-FedEXP), which extend FedEXP to the LDP and CDP setting in a non-trivial way.

\subsection{LDP-FedEXP}
\subsubsection{Naive Implementation of FedEXP with Noisy Updates}
In the setting of LDP, the server can only access the noisy updates $c_i^{(t)}$.
Extending Eq.~\eqref{eq:step-size-fedexp} to the DP setting naively, we obtain the following formula:
\begin{align}
    \tilde \eta_g^{(t)} := \frac{\frac{1}{M}\sum_{i=1}^M \norm{c^{(t)}_i}^2}{\norm{\bar c^{(t)}}^2}. \label{eq:naive}
\end{align}
Unfortunately, as shown in Fig.~\ref{fig:eta_g}, $\tilde \eta_g^{(t)}$ tends to be extremely large and cause instability in the training process.

For simplicity, we focus on the case where the local randomizer is Gaussian mechanism.
To investigate the reason of this phenomenon, let us evaluate the expectation of the numerator in the above formula.
We have
\begin{align*}
    \Expec{\frac{1}{M}\sum_{i=1}^M \norm{c^{(t)}_i}^2} = \frac{1}{M}\sum_{i=1}^M \norm{\Delta^{(t)}_i}^2 + d\sigma^2.
\end{align*}
Since the noise scale $\sigma$ is relatively large in the LDP setting,
the noise term $d\sigma^2$ dominates the numerator.
Furthermore, since the noise term does not depend on the number of clients $M$,
increasing the number of clients does not help to stabilize the training.

\subsubsection{Step Size Formula for Gaussian Mechanism}
To develop a practical step size rule in the DP setting,
let us consider the following \textit{approximate projection condition}:
\begin{align}
    \frac{1}{M} \sum_{i=1}^M \norm{w_i^{(t, \tau)} - w^*}^2 = (1 - \alpha) \norm{w^{(t)} - w^*}^2, \label{eq:proj-cond}
\end{align}
for some $0 \leq \alpha \leq 1$~\citep{jhunjhunwala2023fedexp}, where $w^*$ is a optimal solution of problem~\eqref{eq:problem}. Intuitively, this condition implies that the parameters of the local models are closer to the optimal solution on average after $\tau$ steps of local training.
Note that we consider the condition to motivate our proposed step size, and we prove the convergence guarantee under much milder conditions in Section~\ref{sec:utility}.
Under the above condition, the distance between updated model and the optimal model is evaluated as
\begin{align*}
     & \norm{w^{(t + 1)} - w^*}^2 \simeq (1 - \alpha \eta_g)\norm{w^{(t)} - w^*}^2                       \\
     & \quad - \eta_g \frac{1}{M} \sum_{i=1}^M \norm{\Delta^{(t)}_i}^2 + \eta_g^2 \norm{\bar c^{(t)}}^2,
\end{align*}
for sufficiently large $d$ with high-probability. Here, we ignore the effect of clipping for simplicity.
See Lemma~\ref{lemma:decrease} for the detailed derivation.
To ensure that the distance between the global model and the optimal model decreases for any $\norm{w^{(t)} - w^*}^2$,
we need to set the global step size as
\begin{align}
    \eta_g \leq \eta^{(t)}_{\mathrm{target}} := \frac{\frac{1}{M} \sum_{i=1}^M \norm{\Delta^{(t)}_i}^2}{\norm{\bar c^{(t)}}^2} \label{eq:optimal-eta-g}
\end{align}
but we cannot compute $\eta^{(t)}_{\mathrm{target}}$ since the server cannot access $\Delta_i^{(t)}$ directly.
Instead of the exact calculation of $\frac{1}{M}\sum_{i=1}^M \|\Delta^{(t)}_i\|^2$, we propose to use its unbiased estimator $\frac{1}{M}\sum_{i=1}^M \|c^{(t)}_i\|^2 - d\sigma^2$.
That is, the global step size for LDP-FedEXP is given by
\begin{align}
    \eta_g^{(t)} := \max\ab\{1, \frac{\frac{1}{M}\sum_{i=1}^M \norm{c^{(t)}_i}^2 - d \sigma^2}{\norm{\bar c^{(t)}}^2}\}. \label{eq:step-size-gaussian}
\end{align}
Here, we take the maximum of 1 and the bias-corrected step size to ensure the acceleration of the convergence.
As shown in Fig.~\ref{fig:eta_g}, $\eta_g^{(t)}$ is close to $\eta_{\mathrm{target}}^{(t)}$ for large $M$.
Using the above formula, LDP-FedEXP updates the global model as
$
    w^{(t + 1)} := w^{(t)} + \eta_g^{(t)} \bar c^{(t)}.
$
We show the entire training process in Algorithm~\ref{alg:ldp-fedexp}.

\begin{remark}[Adaptivity to the Noise Scale]
    The expectation of denominator in the step size rule $\mathbb{E}[\|\bar c^{(t)}\|^2]$ is given by $\|\bar \Delta^{(t)}\|^2 + d\sigma^2 / M$.
    Here, $d\sigma^2 / M$ represents the effective noise scale which is added to $\bar c^{(t)}$.
    Thus, the step size is small if the noise scale $\sigma$ is large or the number of clients $M$ is small.
    Indeed, Fig.~\ref{fig:eta_g} shows that our proposed step size increases as the number of clients $M$ increases.
    That is, the step size is adaptive not only to the heterogeneity of the local updates but also to the effective noise scale.
\end{remark}

\subsubsection{Step Size Formula for PrivUnit}\label{sec:privunit}
In the previous section, we have provided the step size formula for Gaussian mechanism.
Here, we provide the step size rule for PrivUnit.

Let $\hat r_i^{(t)} = \scalardp(\Delta_i^{(t)};\varepsilon_2)$ and $z_i^{(t)} = \privunit(\Delta_i^{(t)} / \|\Delta_i^{(t)}\|;\varepsilon_0, \varepsilon_1)$.
Note that $c_i^{(t)} = \hat r_i^{(t)} \cdot z_i^{(t)}$.
Since $\norm{z_i} = 1 / m$, where $m > 0$ is a constant,
we can calculate $|\hat r_i^{(t)}|$ as $m \cdot \|c_i^{(t)}\|$.
Furthermore, since $\hat r_i^{(t)}$ takes discrete values, we can reconstruct $\hat r^{(t)}_i$ from $|\hat r^{(t)}_i|$
except for special choices of privacy parameter $\varepsilon_2$.
However, as shown in~\citet{bhowmick2018protection}, the variance of the noisy update is not constant
and depends on the norm of the original update in a complicated way.
Thus, it is not straightforward to develop an unbiased estimator of $\|\Delta_i^{(t)}\|^2$.
To deal with this issue, we utilize the following upper bound of the variance of PrivUnit:
\begin{align*}
    \Expec{\ab(\hat r_i^{(t)} - r_i^{(t)})^2} & \leq c_1 \ab(r_i^{(t)})^2 + c_2 r_i^{(t)} + c_3,
\end{align*}
where $r_i^{(t)} = \|\Delta^{(t)}_i\|$, and $c_1, c_2, c_3$ are constants defined in Algorithm~\ref{alg:norm-privunit}.
Based on the above upper bound, we propose the following formula for the step size:
\begin{align}
    \eta_g^{(t)} & = \max\ab\{1, \frac{\frac{1}{M} \sum_{i=1}^M \hat s_i}{\norm{\bar c^{(t)}}^2}\}, \label{eq:step-size-privunit}
\end{align}
where $\hat s_i = \frac{(\hat r_i^{(t)})^2 - c_2 \hat r_i^{(t)} - c_3}{1 + c_1}$.
See Algorithm~\ref{alg:norm-privunit} for the detailed procedure.
Here, $\frac{1}{M} \sum_{i=1}^M \hat s_i$ is not an unbiased estimator of $\frac{1}{M} \sum_{i=1}^M \|\Delta_i^{(t)}\|^2$
but it satisfies
\begin{align*}
    \Expec{\frac{1}{M} \sum_{i=1}^M \hat s_i} \leq \frac{1}{M}\sum_{i=1}^M \norm{\Delta_i^{(t)}}^2.
\end{align*}
This property is sufficient to prove the convergence guarantee in Section~\ref{sec:utility}.
In addition, as shown in Fig.~\ref{fig:eta_g}, the step size formula~\eqref{eq:step-size-privunit}
accurately estimates $\eta^{(t)}_{\mathrm{target}}$.

\subsection{CDP-FedEXP}
In the CDP setting, the server can calculate Eq.~\eqref{eq:optimal-eta-g} but it does not satisfy DP.
Since $\|\bar c^{(t)}\|$ can be arbitrarily small and the sensitivity of $\eta_{\mathrm{target}}^{(t)}$ is not bounded, we cannot apply Gaussian mechanism to Eq.~\eqref{eq:optimal-eta-g} directly.
Thus, we propose the following formula:
\begin{align}
    \eta_g^{(t)} := \max\ab\{1, \frac{\frac{1}{M}\sum_{i=1}^M \norm{\Delta_i^{(t)}}^2 + \xi^{(t)}}{\norm{\bar c^{(t)}}^2}\}, \label{eq:step-size-cdp}
\end{align}
where $\xi^{(t)}$ follows $\mathcal{N}(0, \sigma_\xi^2)$.
Here, the numerator is an unbiased estimator of $\frac{1}{M}\sum_{i=1}^M \|\Delta_i^{(t)}\|^2$.
We show the entire training process in Algorithm~\ref{alg:cdp-fedexp}.

Since clipping at the client side ensures that $\|\Delta_i^{(t)}\|^2 \leq C^2$
the sensitivity of the numerator is bounded by $C^2 / M$.
Thus, the above formula satisfies the CDP.
The variance $\sigma_\xi$ of $\xi^{(t)}$ seems to be a hyperparameter
but we can set $\sigma_\xi$ sufficiently small without degrading the privacy guarantee
if $d$ is large since the privacy budget for privatizing the scalar is negligible compared to that for privatizing the $d$-dimensional vector $\bar \Delta^{(t)}$.
Moreover, we find that it is sufficient to set $\sigma_\xi = d\sigma^2 / M$
to obtain the same bias from DP noise as DP-FedAvg based on the convergence analysis.
This makes the step size formula completely hyperparameter-free.

\begin{algorithm}[tbh]
    \caption{LDP-FedEXP}
    \label{alg:ldp-fedexp}
    \begin{algorithmic}
        \STATE {\bfseries Input:} initial $w^{(0)}$, clipping threshold $C$, number of rounds $T$
        \STATE {\bfseries Output:} final $w^{(T)}$
        \FOR{$t=1$ {\bfseries to} $T$}
        \STATE Server sends $w^{(t-1)}$ to all clients
        \FOR{client $i=1$ {\bfseries to} $M$}
        \STATE $\tilde \Delta_i^{(t)} \leftarrow \mathrm{localupdate}(w^{(t - 1)}, \mathcal{D}_i)$
        \STATE $\Delta_i^{(t)} \leftarrow \min\{C / \|\tilde \Delta_i^{(t)}\|, 1\} \cdot \tilde \Delta_i^{(t)}$
        \STATE $c_i^{(t)} \leftarrow \mathrm{LocalRandomizer}(\Delta_i^{(t)})$
        \STATE Client $i$ sends $c_i^{(t)}$ to server
        \ENDFOR
        \STATE Aggregate local updates: $\bar c^{(t)} \leftarrow \frac{1}{M} \sum_{i=1}^M c^{(t)}_i$
        \STATE Compute global step size $\eta_g^{(t)}$ as in Eq.~\eqref{eq:step-size-gaussian} or~\eqref{eq:step-size-privunit}.
        \STATE Update global model with $w^{(t)} \leftarrow w^{(t-1)} + \eta_g^{(t)} \bar c^{(t)}$
        \ENDFOR
    \end{algorithmic}
\end{algorithm}

\begin{algorithm}[tbh]
    \caption{CDP-FedEXP}
    \label{alg:cdp-fedexp}
    \begin{algorithmic}
        \STATE {\bfseries Input:} initial $w^{(0)}$, clipping threshold $C$, noise scale $\sigma$, number of rounds $T$
        \STATE {\bfseries Output:} final $w^{(T)}$
        \FOR{$t=1$ {\bfseries to} $T$}
        \STATE Server sends $w^{(t-1)}$ to all clients
        \FOR{user $i=1$ {\bfseries to} $M$}
        \STATE $\tilde \Delta_i^{(t)} \leftarrow \mathrm{localupdate}(w^{(t - 1)}, \mathcal{D}_i)$
        \STATE $\Delta_i^{(t)} \leftarrow \min\{C / \|\tilde \Delta_i^{(t)}\|, 1\} \cdot \tilde \Delta_i^{(t)}$
        \STATE Client $i$ sends $\Delta_i^{(t)}$ to server
        \ENDFOR
        \STATE Aggregate local updates and add noise: \\
        ~~$\bar c^{(t)} \leftarrow \frac{1}{M} \sum_{i=1}^M \Delta^{(t)}_i + \varepsilon^{(t)}\quad (\varepsilon^{(t)} \sim \mathcal{N}(0, \sigma^2 / M))$
        \STATE Compute global step size $\eta_g^{(t)}$ as in Eq.~\eqref{eq:step-size-cdp}.
        \STATE Update global model with $w^{(t)} \leftarrow w^{(t-1)} + \eta_g^{(t)} \bar c^{(t)}$
        \ENDFOR
    \end{algorithmic}
\end{algorithm}

\begin{algorithm}[tbh]
    \caption{Local update}
    \label{alg:localupdate}
    \begin{algorithmic}
        \STATE {\bfseries Input:} initial $w^{(t, 0)}$, local dataset $\mathcal{D}_i$
        \STATE {\bfseries Output:} final $w^{(t, \tau)}$
        \FOR{$k=1$ {\bfseries to} $\tau$}
        \STATE $w^{(t, k)} \leftarrow w^{(t, k-1)} - \eta_l \nabla f_i(w^{(t, k-1)})$
        \ENDFOR
    \end{algorithmic}
\end{algorithm}

\begin{algorithm}[tbh]
    \caption{Norm Estimation for PrivUnit}
    \label{alg:norm-privunit}
    \begin{algorithmic}
        \STATE {\bfseries Input:} Noisy update $c := \privunit(\Delta / \norm{\Delta};\varepsilon_0, \varepsilon_1) \cdot \scalardp(\norm{\Delta};\varepsilon_2)$
        \STATE {\bfseries Output:} Estimated value $\hat s$ of $\norm{\Delta}^2$
        \STATE Set $a, b, k > 0$ as in Algorithm~\ref{alg:scalardp} and $m$ as in Algorithm~\ref{alg:privunit}
        \STATE $\tilde r \leftarrow m \cdot \norm{c}$, $\tilde J \leftarrow \tilde r / a + b$.
        \STATE \textbf{if} $\tilde J \in \Z$ \textbf{then} $\hat r \leftarrow \tilde r$ \textbf{else} $\hat r \leftarrow -\tilde r$
        \STATE $\hat s \leftarrow \frac{1}{1 + c_1} (\hat r^2 - c_2 \hat r - c_3)$, \\
        where $c_1 = \frac{k + 1}{e^{\varepsilon_2} - 1}, c_2 = -c_1 C, c_3 = (c_1 + 1) \frac{C^2}{4k^2} + c_1 C^2 \ab[\frac{(2k + 1)(e^{\varepsilon_2}+ k)}{6k(e^{\varepsilon_2} - 1)} - \frac{k + 1}{4(e^{\varepsilon_2} - 1)}]$.
    \end{algorithmic}
\end{algorithm}

\section{Theoretical Analysis}
In this section, we provide the privacy guarantee and the convergence analysis of the proposed DP-FedEXP algorithm.
We find that the proposed methods provably accelerate the DP-FedAvg while maintaining the privacy guarantee.

\subsection{Privacy}\label{sec:privacy}
Here, we provide the formal privacy guarantee of LDP-FedEXP and CDP-FedEXP.
\begin{proposition}[LDP case]\label{prop:ldp}
    LDP-FedEXP satisfies the same privacy guarantee as DP-FedAvg in the LDP setting.
    That is, the local computation at each client in LDP-FedEXP with Gaussian mechanism satisfies $(\varepsilon, \delta)$-LDP,
    where $\rho = 2C^2 / \sigma^2$ and $\varepsilon = \alpha \rho + \log (1/\delta) / (\alpha - 1)$ for any $\delta \in (0, 1)$ and $\alpha \in (1, \infty)$.
    In addition, LDP-FedEXP with PrivUnit satisfies $\varepsilon$-LDP,
    where $\varepsilon = \varepsilon_0 + \varepsilon_1 + \varepsilon_2$.
\end{proposition}
\begin{proposition}[CDP case]\label{prop:cdp}
    The entire training process of CDP-FedEXP satisfies $(\varepsilon, \delta)$-CDP,
    where $\rho = 2C^2T / M\sigma^2, \rho_\xi = C^4T / 2M^2 \sigma_\xi^2$
    and $\varepsilon = \alpha(\rho + \rho_\xi) + \log (1/\delta) / (\alpha - 1)$ for any $\delta \in (0, 1)$ and $\alpha \in (1, \infty)$.
\end{proposition}
Our proof for Gaussian mechanism is based on R\`{e}nyi differential privacy (RDP)~\cite{mironov2017renyi} and its composition property.
See Appendix~\ref{appendix:privacy} for details.
For LDP case, the privacy guarantee of LDP-FedEXP is the same as that of LDP-FedAvg
since we use the same mechanism for the local computation.
For CDP case, additional privacy budget $\alpha \rho_\xi$ is required for privatizing the numerator in the step size formula.
However, if we set $\sigma_\xi = d\sigma^2 / M$,
we have $\rho_\xi = C^4 T / 2d^2\sigma^4 = O(\rho^2M^2/ Td^2)$.
Thus, the additional privacy budget consumption is negligible if $\rho = O(1)$ and $T \cdot d^2 \gg M^2$, which is a common setting in modern deep learning tasks.

\subsection{Utility}\label{sec:utility}
In this section, we prove the convergence guarantee of the DP-FedEXP for general non-convex objectives.
Here, we require the following standard assumptions:
\begin{assumption}[Smoothness and Lipschitz continuity]\label{assumption:smoothness}
    Each client loss function $f_i$ is $L$-smooth and $G$-Lipschitz continuous,
    where $L, G > 0$ are constants.
    That is, for any $w, w' \in \R^d$, we have $\norm{\nabla f_i(w) - \nabla f_i(w')} \leq L \norm{w - w'}$ and $\norm{\nabla f_i(w)} \leq G$.
\end{assumption}

\begin{assumption}[Bounded gradient diversity]\label{assumption:diversity}
    For any $w \in \R^d$, the diversity of the gradients is bounded as
    \begin{align*}
        \frac{1}{M} \sum_{i=1}^M \norm{\nabla f_i(w) - \nabla F(w)}^2 & \leq {\sigma_g}^2,
    \end{align*}
    where $\sigma_g^2$ is a constant.
\end{assumption}

Under the above assumptions, we provide the convergence guarantee of LDP-FedEXP and CDP-FedEXP.
\begin{theorem}[LDP case]\label{thm:non-convex-ldp}
    Assume that Assumptions~\ref{assumption:smoothness} and~\ref{assumption:diversity} hold.
    Let $F^* = \min_w F(w)$ and $C = \eta_l \tau G$. Then, for any $\eta_l = \Theta(1/(L\tau)) < 1/(24L\tau)$ and the sequence $\{w^{(t)}\}_{t = 1}^T$ generated by LDP-FedEXP with Gaussian mechanism satisfies
    \begin{align*}
         & \min_{t \in [T]} \norm{\grad F(w^{(t)})}^2 \leq \underbrace{O\ab(\frac{F(w^{0}) - F^*}{\sum_{t=1}^T \eta_g^{(t)}\eta_l \tau})}_{T_1:=\text{initialization error}}        \\                                                                                       \\
         & \quad + \underbrace{O(\eta_l^2L^2\tau(\tau - 1)\sigma_g^2)}_{T_2 := \text{client drift error}} + \underbrace{O(\eta_l L\tau \sigma_g^2)}_{T_3 := \text{global variance}} \\
         & \quad + \underbrace{O\ab(\frac{L\sigma^2q^2}{\eta_l \tau} \ab[\frac{d}{M} + \sqrt{\frac{d}{M}}])}_{T_4^{\text{gauss}} := \text{privacy error}}
    \end{align*}
    with probability at least $1 - Te^{-c \cdot q^2}$ for any $q \in [1, \sqrt{M}]$, where $c$ is a numerical constant.
    On the other hand, LDP-FedEXP with PrivUnit for $\varepsilon_1, \varepsilon_2 = \Omega(1)$ satisfies
    \begin{align*}
        \min_{t \in [T]} \norm{\grad F(w^{(t)})}^2 & \leq T_1 + T_2 + T_3                                                                                                                          \\
                                                   & + \underbrace{O\ab(\frac{L \sigma^2 q^2}{\eta_l \tau} \ab[\frac{d}{M} + \sqrt{\frac{1}{M}}])}_{T^{\text{privunit}}_4 := \text{privacy error}}
    \end{align*}
    with probability at least $1 - Te^{- c \cdot q^2}$ for any $q \in [1, \sqrt{M}]$, where $c$ is a numerical constant.
\end{theorem}
\begin{theorem}[CDP case]\label{thm:non-convex-cdp}
    Assume that Assumptions~\ref{assumption:smoothness} and~\ref{assumption:diversity} hold.
    Let $F^* = \min_w F(w), \sigma_\xi = d\sigma^2 / M$, and $C = \eta_l \tau G$. Then, for any $\eta_l = \Theta(1/(L\tau)) < 1/(24L\tau)$,
    the sequence $\{w^{(t)}\}_{t = 1}^T$ generated by CDP-FedEXP satisfies
    \begin{align*}
        \min_{t \in [T]} \norm{\grad F(w^{(t)})}^2 & \leq T_1 + T_2 + T_3 + \underbrace{O\ab(\frac{L\sigma^2q^2}{\eta_l \tau} \cdot \frac{d}{M})}_{T_4^{\text{cdp}} := \text{privacy error}}
    \end{align*}
    with probability at least $1 - Te^{-c \cdot q^2}$ for $q \in [1, \sqrt{M}]$, where $c$ is a numerical constant
\end{theorem}

See Appendix~\ref{appendix:utility} for the proof. The difficulty of the proof lies in the correlation between the global step size $\eta_g^{(t)}$
and the noisy update $\bar c^{(t)}$ as discussed in previous works~\citep{jhunjhunwala2023fedexp, li2024the}.
Since the step size $\eta_g^{(t)}$ depends on noisy update $\bar c^{(t)}$ in a complicated way, we need to carefully evaluate the error terms from DP noise.

\paragraph{Comparison with FedEXP}
The above theorems imply that the errors of LDP-FedEXP and CDP-FedEXP are decomposed into four terms:
initialization error $T_1$, client drift error $T_2$, global variance $T_3$, and privacy error $T_4$.
As shown in Theorem 2 from~\citet{jhunjhunwala2023fedexp}, the error of FedEXP is given by $T_1 + T_2 + T_3$.
Thus, the DP noise only affects the privacy error term $T_4$, which vanishes as the number of clients $M$ goes to infinity.

\vspace{-0.2cm}
\paragraph{Comparison with DP-FedAvg}
The error of DP-FedAvg is given by
$O\ab(\frac{F(w^{(0)}) - F^*}{T\eta_l \tau}) + O(\eta_l^2L^2\tau^2 \sigma_g^2) + O(\frac{L\sigma^2}{\eta_l \tau}\cdot \frac{d}{M})$
for both LDP and CDP cases~\citep{zhang2022understanding}.
The initialization error term $O\ab(\frac{F(w^{(0)}) - F^*}{T \eta_l \tau})$ is always larger than that of LDP-FedEXP and CDP-FedEXP since $\eta_g^{(t)} \geq 1$ for any $t$.
Thus, DP-FedEXP provably accelerate the convergence of DP-FedAvg in both LDP and CDP setting.
For the privacy error term $T_4$, LDP-FedEXP with the Gaussian mechanism has the additional term of order $\sqrt{\frac{d}{M}}$
unless $d = \Omega(M)$.
In contrast, LDP-FedEXP with PrivUnit achieves the same privacy error as a vanilla DP-FedAvg if $d = \Omega(\sqrt{M})$.
The difference comes from the estimation error of the numerator in the step size formula.
For PrivUnit, we can estimate the squared norm of the local update more accurately due to the separated privatization procedure.
Indeed, the variance of the global step size $\eta_g^{(t)}$ for PrivUnit is much smaller than that of the Gaussian mechanism as shown in Fig.~\ref{fig:eta_g}.
For the CDP case, CDP-FedEXP achieves the same privacy error as DP-FedAvg by setting $\sigma_\xi = d\sigma^2 / M$.


\section{Numerical Experiments}\label{sec:experiments}

\begin{figure*}[t]
    \centering
    \includegraphics[width=\textwidth]{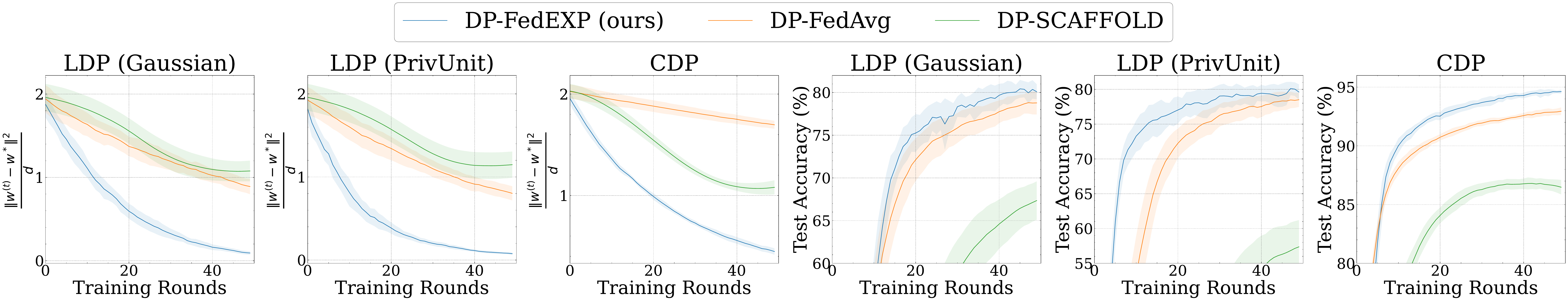}
    \vspace{-0.5cm}
    \caption{The distance to the optimal solution for the synthetic dataset (left) and test accuracy for the MNIST dataset (right).
        In both LDP and CDP cases, DP-FedEXP consistently outperforms baseline algorithms.}
    \label{fig:ldp}
\end{figure*}

\begin{figure}[t]
    \centering
    \includegraphics[width=0.45\textwidth]{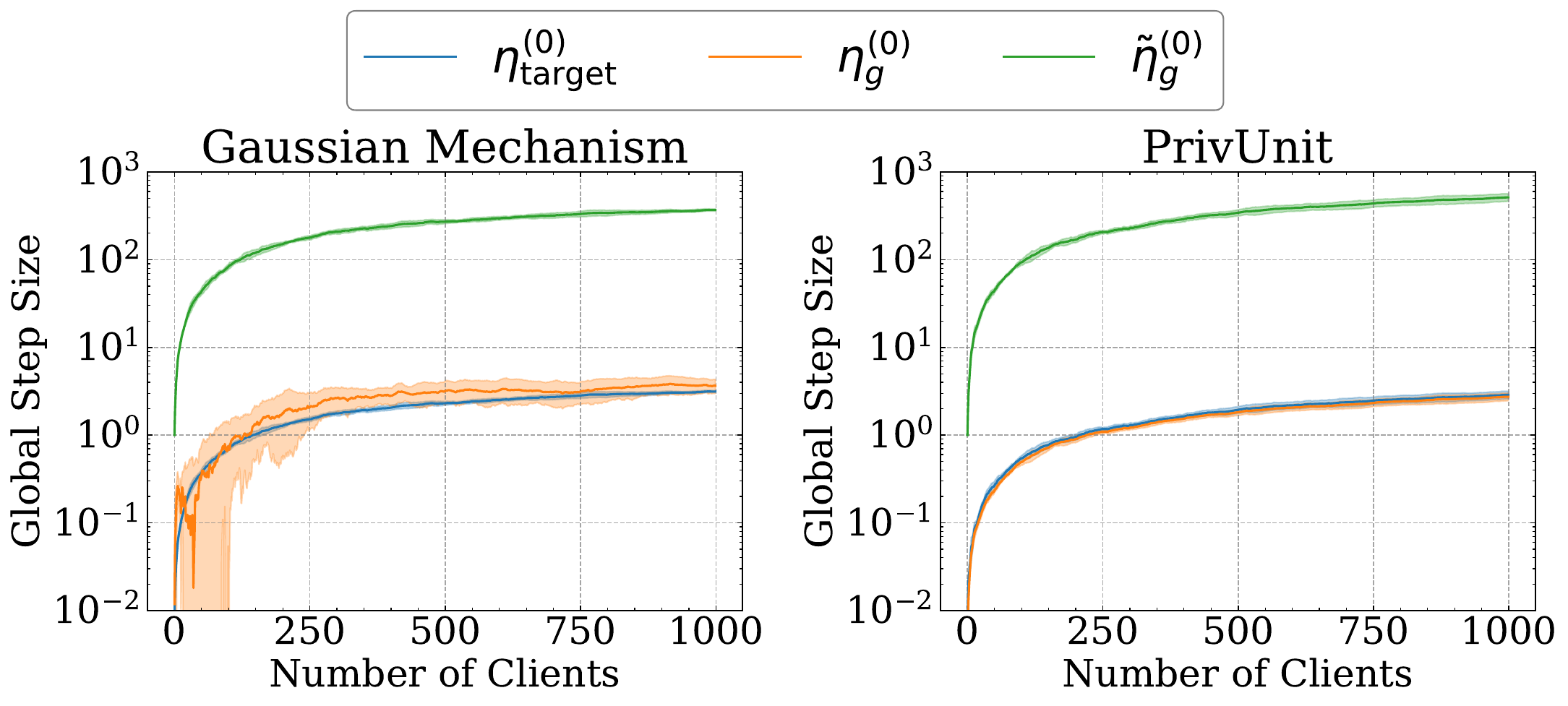}
    \vspace{-0.2cm}
    \caption{The adaptive step size $\eta_g^{(0)}$ at initialization in the LDP setting. Our proposed step size is close to the target step size $\eta^{(0)}_{\mathrm{target}}$ for large $M$
    while the naive step size $\tilde \eta_g^{(0)}$ is extremely large due to the bias in the numerator and the error does not decrease as $M$ increases.}
    \label{fig:eta_g}
\end{figure}


In this section, we evaluate the performance of DP-FedEXP on synthetic and real datasets.
For the synthetic experiment, we consider a linear regression problem, where clients share the common minimizer.
As shown in~\citet{jhunjhunwala2023fedexp}, this setting satisfies the approximate projection condition~\eqref{eq:proj-cond}
and allows us to analyze the convergence of the proposed method.
For the realistic experiment, we consider the image classification task on the MNIST dataset~\citep{lecun1998mnist}.
We compare our proposed method with the baseline algorithms such as DP-FedAvg and DP-SCAFFOLD.
Our framework can be combined with adaptive clipping~\citep{andrew2021differentially} but we use a fixed clipping threshold for simplicity.
For fair comparison, we have tuned the clipping threshold $C$ and the local learning rate $\eta_l$ for each method via grid search.
In both experiments, we run the training for $T = 50$ rounds and set $\sigma = 5 \cdot C / \sqrt{M}, \sigma_\xi = d\sigma^2 / M$ for the CDP case, $\sigma = 0.7 \cdot C$ for the LDP (Gaussian) case, and $\varepsilon_0 = \varepsilon_1 = \varepsilon_2 = 2$ for the LDP (PrivUnit) case.
Following~\citet{jhunjhunwala2023fedexp}, we set the final model as the average of the last 2 iterates to mitigate the effect of oscillating behavior of DP-FedEXP.
For privacy analysis, we utilized the numerical composition~\citep{gopi2021numerical} to tightly audit the privacy leakage.
See Appendix~\ref{appendix:experiments} for the detailed setup and additional results.

\vspace{-0.2cm}
\paragraph{Synthetic Experiment Setup}
First, we generate the target vector $w^* \in \R^d$ according to the standard normal distribution, which is shared among all clients.
Then, we generate the local dataset following a similar procedure in~\citet{li2020federated,jhunjhunwala2023fedexp} with $M = 1000$.
In this experiment, we set $\tau = 20$.
For the CDP setting, we set $d = 500$ while $d = 100$ for the LDP setting since the noise level of LDP is much larger than that of CDP.

\paragraph{Realistic Experiment Setup}
We divide the training data into $M = 1000$ clients according to Dirichlet distribution with $\alpha = 0.3$, following the procedure in~\citet{hsu2019measuring}.
In this experiment, we set $\tau=10$.
For the CDP setting, we use a simple convolutional neural network (CNN) model with two convolutional layers and two fully connected layers.
For LDP setting, we use a small CNN model with two convolutional layers and one fully connected layer.

\begin{table}[t]
    \centering
    \caption{Comparison of the privacy budget $\varepsilon$ for DP-FedEXP and DP-FedAvg.
        We set $\delta = 10^{-5}$ for Gaussian mechanism.
        LDP guarantee is the same for both synthetic and MNIST experiments.
    }
    \label{tab:comparison}
    \begin{tabular}{lcccc}
        \toprule
        Problem setting & DP-FedEXP & DP-FedAvg \\      \midrule
        LDP (Gaussian)  & $15.659$  & $15.659$  \\
        LDP (PrivUnit)  & $6$       & $6$       \\
        CDP (Synthetic) & $15.647$  & $15.258$  \\
        CDP (MNIST)     & $15.261$  & $15.258$  \\
        \bottomrule
    \end{tabular}
\end{table}

\vspace{-0.2cm}
\paragraph{DP-FedEXP consistently outperforms baselines}
Fig.~\ref{fig:ldp} illustrates the mean and standard deviation of the distance to the optimum $w^*$ for the synthetic experiment and the test accuracy for the MNIST experiment over 5 runs with different random seeds.
As discussed in Section~\ref{sec:utility},
DP-FedEXP is expected to converge faster than DP-FedAvg.
Indeed, Fig.~\ref{fig:ldp} illustrates that DP-FedEXP effectively accelerates DP-FedAvg.
In addition, as shown in Table~\ref{tab:comparison},
our proposed methods achieve the same privacy guarantee as DP-FedAvg in the LDP setting and the additional privacy budget in the CDP setting is negligible.
Furthermore, DP-FedEXP consistently outperforms DP-SCAFFOLD.
In our setup, DP-SCAFFOLD does not improve the performance compared to DP-FedAvg except for the case of CDP in the synthetic experiment.
One possible reason is that DP-SCAFFOLD in~\citet{noble2022differentially} is designed for sample-level DP
and the noise scale for client-level DP is much larger than that for sample-level DP.

\vspace{-0.2cm}
\paragraph{The Effect of Bias Correction}
To show the effectiveness of our bias correction scheme in LDP-FedEXP,
we compare the naive step size $\tilde \eta^{(t)}_g$ and the proposed step size $\eta^{(t)}_g$ in Fig.~\ref{fig:eta_g}.
Apparently, the naive step size is extremely large compared to $\eta_{\mathrm{target}}^{(t)}$ in Eq.~\eqref{eq:optimal-eta-g}
and the error does not decrease as the number of clients $M$ increases.
In contrast, the proposed step size is close to $\eta_{\mathrm{target}}^{(t)}$ for large $M$.
In addition, the variance of $\eta_g^{(t)}$ for PrivUnit is much smaller than that for the Gaussian mechanism,
which matches the theoretical analysis in Section~\ref{sec:utility}.

\section{Conclusion}
In this study, we have pursued a practical federated learning framework with formal privacy guarantee.
To this end, we have proposed DP-FedEXP for both LDP and CDP settings,
which adaptively selects the global step size in DP-FL with respect to the heterogeneity of the local updates.
Our proposed framework does not require any additional hyperparameters, additional communication cost or additional computational cost at clients.
Then, we have proved differential privacy guarantee and provided the convergence analysis of our proposed methods.
We have shown that DP-FedEXP provably accelerates DP-FedAvg while maintaining the privacy guarantee.
Finally, we have shown that our proposed methods outperform existing DP-FL algorithms in the numerical experiments.


\bibliography{main}

\begin{thebibliography}{45}
\providecommand{\natexlab}[1]{#1}
\providecommand{\url}[1]{\texttt{#1}}
\expandafter\ifx\csname urlstyle\endcsname\relax
  \providecommand{\doi}[1]{doi: #1}\else
  \providecommand{\doi}{doi: \begingroup \urlstyle{rm}\Url}\fi

\bibitem[Andrew et~al.(2021)Andrew, Thakkar, McMahan, and Ramaswamy]{andrew2021differentially}
Andrew, G., Thakkar, O., McMahan, B., and Ramaswamy, S.
\newblock Differentially private learning with adaptive clipping.
\newblock In \emph{Advances in Neural Information Processing Systems}, volume~34, pp.\  17455--17466, 2021.

\bibitem[Anonymous(2024)]{anonymous2024towards}
Anonymous.
\newblock Towards hyperparameter-free optimization with differential privacy.
\newblock In \emph{Submitted to The Thirteenth International Conference on Learning Representations}, 2024.
\newblock URL \url{https://openreview.net/forum?id=2kGKsyhtvh}.
\newblock under review.

\bibitem[Asi et~al.(2022)Asi, Feldman, and Talwar]{asi2022optimal}
Asi, H., Feldman, V., and Talwar, K.
\newblock Optimal algorithms for mean estimation under local differential privacy.
\newblock In \emph{International Conference on Machine Learning}, pp.\  1046--1056, 2022.

\bibitem[Bhowmick et~al.(2018)Bhowmick, Duchi, Freudiger, Kapoor, and Rogers]{bhowmick2018protection}
Bhowmick, A., Duchi, J., Freudiger, J., Kapoor, G., and Rogers, R.
\newblock Protection against reconstruction and its applications in private federated learning.
\newblock \emph{arXiv preprint arXiv:1812.00984}, 2018.

\bibitem[Bu et~al.(2023)Bu, Wang, Zha, and Karypis]{bu2024automatic}
Bu, Z., Wang, Y.-X., Zha, S., and Karypis, G.
\newblock Automatic clipping: Differentially private deep learning made easier and stronger.
\newblock In Oh, A., Naumann, T., Globerson, A., Saenko, K., Hardt, M., and Levine, S. (eds.), \emph{Advances in Neural Information Processing Systems}, volume~36, pp.\  41727--41764, 2023.

\bibitem[Daigavane et~al.(2022)Daigavane, Madan, Sinha, Thakurta, Aggarwal, and Jain]{daigavane2022node}
Daigavane, A., Madan, G., Sinha, A., Thakurta, A.~G., Aggarwal, G., and Jain, P.
\newblock Node-level differentially private graph neural networks.
\newblock In \emph{ICLR 2022 Workshop on PAIR{\textasciicircum}2Struct: Privacy, Accountability, Interpretability, Robustness, Reasoning on Structured Data}, 2022.

\bibitem[Dwork et~al.(2006)Dwork, McSherry, Nissim, and Smith]{dwork2006calibrating}
Dwork, C., McSherry, F., Nissim, K., and Smith, A.
\newblock Calibrating noise to sensitivity in private data analysis.
\newblock In \emph{Theory of Cryptography Conference}, pp.\  265--284, 2006.

\bibitem[Dwork et~al.(2014)Dwork, Roth, et~al.]{dwork2014algorithmic}
Dwork, C., Roth, A., et~al.
\newblock The algorithmic foundations of differential privacy.
\newblock \emph{Foundations and Trends{\textregistered} in Theoretical Computer Science}, 9\penalty0 (3--4):\penalty0 211--407, 2014.

\bibitem[Geiping et~al.(2020)Geiping, Bauermeister, Dr{\"o}ge, and Moeller]{geiping2020inverting}
Geiping, J., Bauermeister, H., Dr{\"o}ge, H., and Moeller, M.
\newblock Inverting gradients-how easy is it to break privacy in federated learning?
\newblock In \emph{Advances in neural information processing systems}, volume~33, pp.\  16937--16947, 2020.

\bibitem[Gopi et~al.(2021)Gopi, Lee, and Wutschitz]{gopi2021numerical}
Gopi, S., Lee, Y.~T., and Wutschitz, L.
\newblock Numerical composition of differential privacy.
\newblock In \emph{Advances in Neural Information Processing Systems}, volume~34, pp.\  11631--11642, 2021.

\bibitem[Gross(2011)]{gross2011recovering}
Gross, D.
\newblock Recovering low-rank matrices from few coefficients in any basis.
\newblock \emph{IEEE Transactions on Information Theory}, 57\penalty0 (3):\penalty0 1548--1566, 2011.

\bibitem[Haddadpour \& Mahdavi(2019)Haddadpour and Mahdavi]{haddadpour2019convergence}
Haddadpour, F. and Mahdavi, M.
\newblock On the convergence of local descent methods in federated learning.
\newblock \emph{arXiv preprint arXiv:1910.14425}, 2019.

\bibitem[Hsu et~al.(2019)Hsu, Qi, and Brown]{hsu2019measuring}
Hsu, T.-M.~H., Qi, H., and Brown, M.
\newblock Measuring the effects of non-identical data distribution for federated visual classification.
\newblock \emph{arXiv preprint arXiv:1909.06335}, 2019.

\bibitem[Jhunjhunwala et~al.(2023)Jhunjhunwala, Wang, and Joshi]{jhunjhunwala2023fedexp}
Jhunjhunwala, D., Wang, S., and Joshi, G.
\newblock Fedexp: Speeding up federated averaging via extrapolation.
\newblock In \emph{International Conference on Learning Representations}, 2023.

\bibitem[Kairouz et~al.(2021)Kairouz, McMahan, Avent, Bellet, Bennis, Bhagoji, Bonawitz, Charles, Cormode, Cummings, et~al.]{kairouz2021advances}
Kairouz, P., McMahan, H.~B., Avent, B., Bellet, A., Bennis, M., Bhagoji, A.~N., Bonawitz, K., Charles, Z., Cormode, G., Cummings, R., et~al.
\newblock Advances and open problems in federated learning.
\newblock \emph{Foundations and trends{\textregistered} in machine learning}, 14\penalty0 (1--2):\penalty0 1--210, 2021.

\bibitem[Karimireddy et~al.(2019)Karimireddy, Rebjock, Stich, and Jaggi]{karimireddy2019error}
Karimireddy, S.~P., Rebjock, Q., Stich, S., and Jaggi, M.
\newblock Error feedback fixes signsgd and other gradient compression schemes.
\newblock In \emph{International Conference on Machine Learning}, pp.\  3252--3261, 2019.

\bibitem[Karimireddy et~al.(2020{\natexlab{a}})Karimireddy, Jaggi, Kale, Mohri, Reddi, Stich, and Suresh]{karimireddy2020mime}
Karimireddy, S.~P., Jaggi, M., Kale, S., Mohri, M., Reddi, S.~J., Stich, S.~U., and Suresh, A.~T.
\newblock Mime: Mimicking centralized stochastic algorithms in federated learning.
\newblock \emph{arXiv preprint arXiv:2008.03606}, 2020{\natexlab{a}}.

\bibitem[Karimireddy et~al.(2020{\natexlab{b}})Karimireddy, Kale, Mohri, Reddi, Stich, and Suresh]{karimireddy2020scaffold}
Karimireddy, S.~P., Kale, S., Mohri, M., Reddi, S., Stich, S., and Suresh, A.~T.
\newblock Scaffold: Stochastic controlled averaging for federated learning.
\newblock In \emph{International conference on machine learning}, pp.\  5132--5143, 2020{\natexlab{b}}.

\bibitem[Kasiviswanathan et~al.(2011)Kasiviswanathan, Lee, Nissim, Raskhodnikova, and Smith]{kasiviswanathan2011can}
Kasiviswanathan, S.~P., Lee, H.~K., Nissim, K., Raskhodnikova, S., and Smith, A.
\newblock What can we learn privately?
\newblock \emph{SIAM Journal on Computing}, 40\penalty0 (3):\penalty0 793--826, 2011.

\bibitem[Kingma \& Ba(2015)Kingma and Ba]{kingma2015adam}
Kingma, D.~P. and Ba, J.
\newblock {Adam: A Method for Stochastic Optimization}.
\newblock In \emph{International Conference on Learning Representations}, 2015.

\bibitem[Konečný et~al.(2017)Konečný, McMahan, Yu, Richtárik, Suresh, and Bacon]{konečný2017federated}
Konečný, J., McMahan, H.~B., Yu, F.~X., Richtárik, P., Suresh, A.~T., and Bacon, D.
\newblock Federated learning: Strategies for improving communication efficiency, 2017.
\newblock URL \url{https://arxiv.org/abs/1610.05492}.

\bibitem[Lam et~al.(2021)Lam, Wei, Brooks, Reddi, and Mitzenmacher]{lam2021gradient}
Lam, M., Wei, G.-Y., Brooks, D., Reddi, V.~J., and Mitzenmacher, M.
\newblock Gradient disaggregation: Breaking privacy in federated learning by reconstructing the user participant matrix.
\newblock In \emph{International Conference on Machine Learning}, pp.\  5959--5968. PMLR, 2021.

\bibitem[LeCun(1998)]{lecun1998mnist}
LeCun, Y.
\newblock The mnist database of handwritten digits.
\newblock \emph{http://yann.lecun.com/exdb/mnist/}, 1998.

\bibitem[Li et~al.(2024)Li, Acharya, and Richt{\'a}rik]{li2024the}
Li, H., Acharya, K., and Richt{\'a}rik, P.
\newblock The power of extrapolation in federated learning.
\newblock In \emph{The Thirty-eighth Annual Conference on Neural Information Processing Systems}, 2024.

\bibitem[Li et~al.(2020)Li, Sahu, Zaheer, Sanjabi, Talwalkar, and Smith]{li2020federated}
Li, T., Sahu, A.~K., Zaheer, M., Sanjabi, M., Talwalkar, A., and Smith, V.
\newblock Federated optimization in heterogeneous networks.
\newblock In \emph{Proceedings of Machine learning and systems}, volume~2, pp.\  429--450, 2020.

\bibitem[Li et~al.(2021)Li, Tramer, Liang, and Hashimoto]{li2021large}
Li, X., Tramer, F., Liang, P., and Hashimoto, T.
\newblock Large language models can be strong differentially private learners.
\newblock \emph{arXiv preprint arXiv:2110.05679}, 2021.

\bibitem[Liu \& Talwar(2019)Liu and Talwar]{liu2019private}
Liu, J. and Talwar, K.
\newblock Private selection from private candidates.
\newblock In \emph{Proceedings of the 51st Annual ACM SIGACT Symposium on Theory of Computing}, pp.\  298--309, 2019.

\bibitem[McMahan et~al.(2017{\natexlab{a}})McMahan, Moore, Ramage, Hampson, and y~Arcas]{mcmahan2017communication}
McMahan, B., Moore, E., Ramage, D., Hampson, S., and y~Arcas, B.~A.
\newblock Communication-efficient learning of deep networks from decentralized data.
\newblock In \emph{International Conference on Artificial Intelligence and Statistics}, pp.\  1273--1282, 2017{\natexlab{a}}.

\bibitem[McMahan et~al.(2017{\natexlab{b}})McMahan, Ramage, Talwar, and Zhang]{mcmahan2017learning}
McMahan, H.~B., Ramage, D., Talwar, K., and Zhang, L.
\newblock Learning differentially private recurrent language models.
\newblock \emph{arXiv preprint arXiv:1710.06963}, 2017{\natexlab{b}}.

\bibitem[Mironov(2017)]{mironov2017renyi}
Mironov, I.
\newblock R{\'e}nyi differential privacy.
\newblock In \emph{IEEE 30th Computer Security Foundations symposium}, pp.\  263--275, 2017.

\bibitem[Mitra et~al.(2021)Mitra, Jaafar, Pappas, and Hassani]{mitra2021linear}
Mitra, A., Jaafar, R., Pappas, G.~J., and Hassani, H.
\newblock Linear convergence in federated learning: Tackling client heterogeneity and sparse gradients.
\newblock In \emph{Advances in Neural Information Processing Systems}, volume~34, pp.\  14606--14619, 2021.

\bibitem[Mohapatra et~al.(2022)Mohapatra, Sasy, He, Kamath, and Thakkar]{mohapatra2022role}
Mohapatra, S., Sasy, S., He, X., Kamath, G., and Thakkar, O.
\newblock The role of adaptive optimizers for honest private hyperparameter selection.
\newblock In \emph{Proceedings of the AAAI conference on artificial intelligence}, volume~36, pp.\  7806--7813, 2022.

\bibitem[Nasr et~al.(2019)Nasr, Shokri, and Houmansadr]{nasr2019comprehensive}
Nasr, M., Shokri, R., and Houmansadr, A.
\newblock Comprehensive privacy analysis of deep learning: Passive and active white-box inference attacks against centralized and federated learning.
\newblock In \emph{IEEE Symposium on Security and Privacy}, pp.\  739--753, 2019.

\bibitem[Noble et~al.(2022)Noble, Bellet, and Dieuleveut]{noble2022differentially}
Noble, M., Bellet, A., and Dieuleveut, A.
\newblock Differentially private federated learning on heterogeneous data.
\newblock In \emph{International Conference on Artificial Intelligence and Statistics}, pp.\  10110--10145, 2022.

\bibitem[Papernot \& Steinke(2021)Papernot and Steinke]{papernot2021hyperparameter}
Papernot, N. and Steinke, T.
\newblock Hyperparameter tuning with renyi differential privacy.
\newblock \emph{arXiv preprint arXiv:2110.03620}, 2021.

\bibitem[Pierra(1984)]{pierra2011decomposition}
Pierra, G.
\newblock Decomposition through formalization in a product space.
\newblock \emph{Mathematical Programming}, 28\penalty0 (1):\penalty0 96--115, 1984.

\bibitem[Reddi et~al.(2021)Reddi, Charles, Zaheer, Garrett, Rush, Kone{\v{c}}n{\'y}, Kumar, and McMahan]{reddi2021adaptive}
Reddi, S.~J., Charles, Z., Zaheer, M., Garrett, Z., Rush, K., Kone{\v{c}}n{\'y}, J., Kumar, S., and McMahan, H.~B.
\newblock Adaptive federated optimization.
\newblock In \emph{International Conference on Learning Representations}, 2021.

\bibitem[Shi et~al.(2023)Shi, Liu, Wei, Shen, Wang, and Tao]{shi2023make}
Shi, Y., Liu, Y., Wei, K., Shen, L., Wang, X., and Tao, D.
\newblock Make landscape flatter in differentially private federated learning.
\newblock In \emph{Proceedings of the IEEE/CVF Conference on Computer Vision and Pattern Recognition}, pp.\  24552--24562, 2023.

\bibitem[Tang et~al.(2024)Tang, Shpilevskiy, and L{\'e}cuyer]{tang2024dp}
Tang, Q., Shpilevskiy, F., and L{\'e}cuyer, M.
\newblock {DP-AdamBC: Your DP-Adam Is Actually DP-SGD (Unless You Apply Bias Correction)}.
\newblock In \emph{Proceedings of the AAAI Conference on Artificial Intelligence}, volume~38, pp.\  15276--15283, 2024.

\bibitem[Wainwright(2019)]{wainwright2019high}
Wainwright, M.~J.
\newblock \emph{High-dimensional statistics: A non-asymptotic viewpoint}, volume~48.
\newblock Cambridge university press, 2019.

\bibitem[Wang et~al.(2023)Wang, Gao, Zhang, Su, and Shen]{wang2023dp}
Wang, H., Gao, S., Zhang, H., Su, W.~J., and Shen, M.
\newblock Dp-hypo: an adaptive private hyperparameter optimization framework.
\newblock \emph{arXiv preprint arXiv:2306.05734}, 2023.

\bibitem[Wang et~al.(2020)Wang, Liu, Liang, Joshi, and Poor]{wang2020tackling}
Wang, J., Liu, Q., Liang, H., Joshi, G., and Poor, H.~V.
\newblock Tackling the objective inconsistency problem in heterogeneous federated optimization.
\newblock In \emph{Advances in neural information processing systems}, volume~33, pp.\  7611--7623, 2020.

\bibitem[Xu et~al.(2023)Xu, Zhang, Andrew, Choquette-Choo, Kairouz, McMahan, Rosenstock, and Zhang]{xu2023federated}
Xu, Z., Zhang, Y., Andrew, G., Choquette-Choo, C.~A., Kairouz, P., McMahan, H.~B., Rosenstock, J., and Zhang, Y.
\newblock Federated learning of gboard language models with differential privacy.
\newblock \emph{arXiv preprint arXiv:2305.18465}, 2023.

\bibitem[Zhang et~al.(2022)Zhang, Chen, Hong, Wu, and Yi]{zhang2022understanding}
Zhang, X., Chen, X., Hong, M., Wu, Z.~S., and Yi, J.
\newblock Understanding clipping for federated learning: Convergence and client-level differential privacy.
\newblock In \emph{International Conference on Machine Learning}, 2022.

\bibitem[Zhao et~al.(2024)Zhao, Bagchi, Avestimehr, Chan, Chaterji, Dimitriadis, Li, Li, Nourian, and Roth]{zhao2024federated}
Zhao, J.~C., Bagchi, S., Avestimehr, S., Chan, K.~S., Chaterji, S., Dimitriadis, D., Li, J., Li, N., Nourian, A., and Roth, H.~R.
\newblock Federated learning privacy: Attacks, defenses, applications, and policy landscape-a survey.
\newblock \emph{arXiv preprint arXiv:2405.03636}, 2024.

\end{thebibliography}
\bibliographystyle{icml2025}

\newpage
\appendix
\onecolumn

\section{Auxiliary Results} \label{appendix:auxiliary}
\begin{lemma}[Gaussian tail bound]\label{lemma:gaussian-tail}
    Let $X$ be a random variable following the Gaussian distribution $\mathcal{N}(0, \sigma^2)$.
    Then, for any $q > 0$, we have
    \begin{align*}
        X & \leq \sigma q \quad \text{with probability at least } 1 - e^{-q^2 / 2}.
    \end{align*}
\end{lemma}
\begin{proof}
    From the Hoeffding bound~\cite{wainwright2019high}, we obtain
    \begin{align*}
        \Prob{X > t} & \leq e^{-t^2 / 2\sigma^2}.
    \end{align*}
    Setting $t = \sigma q$ completes the proof.
\end{proof}
\begin{lemma}[Tail bound for norm of Gaussian]\label{lemma:chi-squared-tail}
    Let $x_i \in \R^d$ be a random variable following the Gaussian distribution $\mathcal{N}(0, \sigma^2 I_d)$.
    Then, for any $q \geq 1$, we have
    \begin{align*}
        \frac{1}{n} \sum_{i=1}^n \norm{x_i}^2 - d \sigma^2 & \leq \sqrt{\frac{d}{n}} \sigma^2 \cdot q^2 \quad \text{with probability at least } 1 - e^{-q^2 / 8}.
    \end{align*}
\end{lemma}
\begin{proof}
    It is sufficient to consider the case of $\sigma^2 = 1$ by scaling $x_i$ with $1 / \sigma$.
    Since $Z_i := \norm{x_i}^2$ follows the $\chi^2$-distribution with $d$ degrees of freedom, we have
    \begin{align*}
        \Expec{e^{\lambda (Z_i - d)}}
         & = e^{-d \lambda} \cdot \ab[\int e^{\lambda X^2} \frac{1}{\sqrt{2\pi }} e^{-X^2 / 2} \mathrm{d} X]^d \\
         & = e^{-d \lambda} \cdot \ab[\frac{1}{\sqrt{1 - 2\lambda}}]^d                                         \\
         & \leq e^{-2d \lambda^2}      \quad \text{for any } \abs{\lambda} \leq 1/4.
    \end{align*}
    Thus, $\sum_{i=1}^n Z_i$ is subexponential random variable with parameters $(\nu^2, b) = (4dn, 4)$ and satisfies
    \begin{align*}
        \Prob{\sum_{i=1}^n Z_i - dn \geq t} & \leq \begin{cases}
                                                       \exp\ab(-\frac{t^2}{8dn}) & \text{for } t \in (0, dn), \\
                                                       \exp\ab(-\frac{t}{8})     & \text{otherwise}.
                                                   \end{cases}
    \end{align*}
    Setting $t = q^2 \cdot \sqrt{dn}$, we obtain
    \begin{align*}
        \Prob{\frac{1}{n}\sum_{i=1}^n Z_i - d \geq \sqrt{\frac{d}{n}} \cdot q^2}
         & \leq \begin{cases}
                    \exp\ab(-q^4 / 8)       & \text{for } t \in (0, \sqrt{dn}), \\
                    \exp\ab(-\frac{q^2}{8}) & \text{otherwise}.
                \end{cases} \\
         & \leq \exp\ab(-\frac{q^2}{8}) \quad \text{for any } q \geq 1.
    \end{align*}
    This completes the proof.
\end{proof}

\begin{lemma}[Vector Bernstein Inequality]\label{lemma:vec-bernstein}
    Let $x_1, \dots, x_n \in \R^d$ be independent zero-mean random variables.
    Assume that $\norm{x_i} \leq R$ almost surely for any $i$.
    Then, for any $q \in [0, \sqrt{n}]$, we have
    \begin{align*}
        \Prob{\norm{\frac{1}{n}\sum_{i=1}^n x_i} \geq \frac{R(1 + q)}{\sqrt{n}}} \leq \exp\ab(-\frac{q^2}{4}).
    \end{align*}
\end{lemma}
\begin{proof}
    Let $V = \sum_{i=1}^n \Expec{\norm{x_i}^2}$.
    Note that $V \leq nR^2$ since $\norm{x_i} \leq R$ almost surely.
    Then, Theorem 12 in~\citet{gross2011recovering} implies
    \begin{align*}
        \Prob{\norm{\sum_{i=1}^n x_i} \geq \sqrt{n} R + t} & \leq \Prob{\norm{\sum_{i=1}^n x_i} \geq \sqrt{V} + t} \leq \exp\ab(-\frac{t^2}{4V})
    \end{align*}
    for any $t \in [0, V/R]$.
    Setting $t = \sqrt{n}R q$, we obtained
    \begin{align*}
        \Prob{\norm{\frac{1}{n}\sum_{i=1}^n x_i} \geq \frac{R(1 + q)}{\sqrt{n}}} = \Prob{\norm{\sum_{i=1}^n x_i} \geq \sqrt{n}R(1 + q)} \leq \exp\ab(-\frac{nR^2 q^2}{4V}) \leq \exp\ab(-\frac{q^2}{4})
    \end{align*}
    for any $q \in [0, \sqrt{n}]$.
\end{proof}

\begin{lemma}\label{lemma:decrease}
    Assume that the generalized approximate projection condition Eq.~\eqref{eq:proj-cond} holds.
    Then, for any $\eta_g > 0$, we have
    \begin{align*}
        \norm{w^{(t + 1)} - w^*}^2 & = (1 - \alpha \eta_g)\norm{w^{(t)} - w^*}^2 - \eta_g \frac{1}{M} \sum_{i=1}^M \norm{\Delta^{(t)}_i}^2 + \eta_g^2 \norm{\bar c^{(t)}}^2 + O\ab(\frac{\eta_g \cdot \sqrt{\frac{d}{M} \sigma^2} \cdot \norm{w^{(t)} - w^*}}{\sqrt{d}} \cdot q),
    \end{align*}
    with probability at least $1 - e^{-q^2 / 2}$ for any $q > 0$.
\end{lemma}
\begin{proof}
    From the generalized approximate projection condition Eq.~\eqref{eq:proj-cond}, we have
    \begin{align*}
        \frac{1}{M} \sum_{i=1}^n \norm{w^{(t)} + \Delta^{(t)}_i - w^*}^2 & = \norm{w^{(t)} - w^*}^2 + \frac{2}{M} \sum_{i=1}^M \langle w^{(t)} - w^*, \Delta^{(t)}_i\rangle + \frac{1}{M}\sum_{i=1}^n \norm{\Delta^{(t)}_i}^2 \\
                                                                         & = (1 - \alpha) \norm{w^{(t)} - w^*}^2.
    \end{align*}
    This implies
    \begin{align*}
        \frac{2}{M} \sum_{i=1}^M \langle w^{(t)} - w^*, \Delta^{(t)}_i\rangle & =-\alpha \norm{w^{(t)} - w^*}^2 - \frac{1}{M} \sum_{i=1}^M \norm{\Delta^{(t)}_i}^2.
    \end{align*}
    Substituting the above equation, we obtain
    \begin{align*}
        \norm{w^{(t)} + \eta_g c^{(t)} - w^*}^2 & = \norm{w^{(t)} - w^*}^2 + \frac{2 \eta_g}{M} \sum_{i=1}^M\langle \Delta_i^{(t)}, w^{(t)} - w^* \rangle + 2 \eta_g \langle \bar \varepsilon^{(t)}, w^{(t)} - w^* \rangle + \eta_g^2 \norm{\bar c^{(t)}}^2 \\
                                                & =(1 - \alpha \eta_g) \norm{w^{(t)} - w^*}^2 - \frac{\eta_g}{M} \sum_{i=1}^M \norm{\Delta_i^{(t)}}^2 + \eta_g^2 \norm{c^{(t)}}^2 + O\ab(\frac{\eta_g \sigma\norm{w^{(t)} - w^*}}{\sqrt{M}} \cdot q),
    \end{align*}
    with probability at least $1 - e^{-q^2 / 2}$ for any $q > 0$.
    Here, we used the fact that $2 \eta_g \langle \bar \varepsilon^{(t)}, w^{(t)} - w^* \rangle$ follows $\mathcal{N}(0, \eta_g^2 \sigma^2 \norm{w^{(t)} - w^*}^2 / M)$
    and Lemma~\ref{lemma:gaussian-tail}.
    This completes the proof.
\end{proof}

\section{Brief review of PrivUnit}
\begin{algorithm}[tbh]
    \caption{PrivUnit}
    \label{alg:privunit}
    \begin{algorithmic}
        \STATE {\bfseries Input:} $u \in \mathbb{S}^{d-1}, \varepsilon_0, \varepsilon_1 > 0$
        \STATE {\bfseries Output:} Randomized vector $Z \in \mathbb{R}^d$
        \STATE $p \leftarrow \frac{e^{\varepsilon_0}}{1 + e^{\varepsilon_0}}$
        \STATE Select $\gamma$ such that
        \begin{align*}
            \gamma        & \leq \frac{e^{\varepsilon_1} - 1}{e^{\varepsilon_1} + 1} \sqrt{\frac{\pi}{2(d - 1)}},                                            \\
                          & \text{or}                                                                                                                        \\
            \varepsilon_1 & \geq \frac{1}{2} \log d + \log 6 - \frac{d - 1}{2} \log (1 - \gamma^2) + \log \gamma \text{ and } \gamma \geq \sqrt{\frac{2}{d}}
        \end{align*}
        \STATE Draw random vector $V$ according to the following distribution:
        \begin{align*}
            V \leftarrow \begin{cases}
                             \text{uniform on } \{v \in \mathbb{S}^{d - 1} \mid \langle v, u \rangle \geq \gamma \} & \text{w.p. } \gamma, \\
                             \text{uniform on } \{v \in \mathbb{S}^{d - 1} \mid \langle v, u \rangle < \gamma \}    & \text{otherwise}.
                         \end{cases}
        \end{align*}
        \STATE $\alpha \leftarrow \frac{d - 1}{2}, \tau=\frac{1 + \gamma}{2}$, and
        \begin{align*}
            m \leftarrow \frac{(1 - \gamma^2)^\alpha}{2^{d - 2}(d - 1)} \ab[\frac{p}{B(\alpha, \alpha) - B(\tau;\alpha, \alpha)} - \frac{1-p}{B(\tau;\alpha, \alpha)}]
        \end{align*}
        \STATE Rescale $V$ as $Z \leftarrow \frac{1}{m} \cdot V$
    \end{algorithmic}
\end{algorithm}

\begin{algorithm}[tbh]
    \caption{ScalarDP}
    \label{alg:scalardp}
    \begin{algorithmic}
        \STATE {\bfseries Input:} magnitude $r \in [0, C]$, privacy parameter $\varepsilon_2 > 0$
        \STATE {\bfseries Output:} Randomized magnitude $\hat r$
        \STATE $k \leftarrow e^{\ceil{\varepsilon_2 / 3}}$
        \STATE $r_{\max} \leftarrow C$
        \STATE Sample $J \in \{0, \dots, k\}$ according to the following distribution:
        \begin{align*}
            J \leftarrow \begin{cases}
                             \floor{kr/r_{\max}} & \text{w.p. } \ceil{kr/r_{\max}} - kr/r_{\max}, \\
                             \ceil{kr/r_{\max}}  & \text{otherwise}.
                         \end{cases}
        \end{align*}
        Draw randomized response $\hat J$ according to the following distribution:
        \begin{align*}
            \hat J \leftarrow \begin{cases}
                                  J                                                   & \text{w.p. } \frac{e^{\varepsilon_2}}{e^{\varepsilon_2} + k}, \\
                                  \text{uniform on } \{0, \dots, k\} \backslash \{J\} & \text{otherwise}.
                              \end{cases}
        \end{align*}
        \STATE Debias $\hat r$ as $\hat r \leftarrow a(\hat J - b)$, where $a = \ab(\frac{e^{\varepsilon_2} + k}{e^{\varepsilon_2} - 1}) \frac{r_{\max}}{k}$ and $b = \frac{k(k + 1)}{2(e^{\varepsilon_2} + k)}$
    \end{algorithmic}
\end{algorithm}

Here, we briefly explain PrivUnit and ScalarDP algorithms proposed by~\citet{bhowmick2018protection}.
We provide the detailed description of the algorithms in Algorithm~\ref{alg:privunit} and~\ref{alg:scalardp}.
As shown in~\citet{bhowmick2018protection}, the product of PrivUnit and ScalarDP is an unbiased estimator of the original vector and
provide the formal privacy guarantee.
\begin{lemma}\label{lemma:privunit-unbiased}
    For $\varepsilon_0, \varepsilon_1, \varepsilon_2 \in [0, d]$,
    $c = \privunit(\Delta / \norm{\Delta}; \varepsilon_0, \varepsilon_1) \cdot \scalardp(\norm{\Delta};\varepsilon_2)$
    is an unbiased estimator of $\Delta$ if $\norm{\Delta} \leq C$. That is, $E[c] = \Delta$.
    Moreover, $c$ satisfies $(\varepsilon_0 + \varepsilon_1 + \varepsilon_2)$-DP.
\end{lemma}
\begin{proof}
    See Proposition 3 and Lemma 4.1 in~\citet{bhowmick2018protection} for the proof.
\end{proof}

In the following, we prove some properties of PrivUnit and norm estimation procedure in Algorithm~\ref{alg:norm-privunit} for the convergence analysis.
\begin{lemma}\label{lemma:privunit-norm}
    Assume that $\frac{k(k + 1)}{e^{\varepsilon_2} + k} \notin \Z$.
    Then, the estimated value $\hat s$ computed by Algorithm~\ref{alg:norm-privunit} satisfies $E[\hat s] \leq r^2$.
\end{lemma}
\begin{proof}
    First, we show that $\hat r = \scalardp(\norm{\Delta})$.
    From the definition of $c = \privunit(\Delta / \norm{\Delta})\cdot \scalardp(\norm{\Delta})$ and $\norm{\privunit(\Delta / \norm{\Delta})} = 1 / m$,
    we have $\tilde r = \abs{\scalardp(\norm{\Delta})}$.
    If $\scalardp(\norm{\Delta}) < 0$ and $\tilde J \in \Z$,
    $\scalardp(\norm{\Delta}) = -\tilde r$
    and $\hat J = \scalardp(\norm{\Delta}) / a + b = -\tilde r / a + b \in \Z$.
    This implies $\hat J + \tilde J = 2b = \frac{k(k + 1)}{e^{\varepsilon} + k} \in \Z$, which contradicts the assumption.
    Thus, $\tilde J \notin \Z$ and $\hat r = -\tilde r = \scalardp(\norm{\Delta})$ if $\scalardp(\norm{\Delta}) < 0$.
    On the other hand, if $\scalardp(\norm{\Delta}) \geq 0$, $\tilde J = \tilde r / a + b = \scalardp(\norm{\Delta}) / a + b \in \Z$ and $\hat r = \tilde r = \scalardp(\norm{\Delta})$.
    Combining the above arguments, we have $\hat r = \scalardp(\norm{\Delta})$.

    Next, we show that $E[\hat s] \leq r^2$.
    As shown in~\citet{bhowmick2018protection}, the variance of $\hat r$ is bounded as follows:
    \begin{align*}
        \Var{\hat r} & \leq \frac{k + 1}{e^{\varepsilon_2} - 1} \ab[r^2 + \frac{r_{\max}^2}{4k^2} - r r_{\max} + \frac{(2k + 1)(e^{\varepsilon_2} + k)r_{\max}^2}{6k(e^{\varepsilon_2} - 1)} - \frac{(k + 1)r_{\max}^2}{4(e^{\varepsilon_2} - 1)}] + \frac{r_{\max}^2}{4k^2} \\
                     & = c_1 r^2 + c_2 r + c_3.
    \end{align*}
    Thus, we have
    \begin{align*}
        \Expec{\hat s} & = \Expec{\frac{1}{1 + c_1} (\hat r^2 - c_2 \hat r - c_3)}             \\
                       & = \frac{1}{1 + c_2} \ab(r^2 + \Var{\hat r} - c_2 r - c_3)             \\
                       & \leq \frac{1}{1 + c_2} \ab(r^2 + c_1 r^2 + c_2 r + c_3 - c_2 r - c_3) \\
                       & = r^2.
    \end{align*}
    This completes the proof.
\end{proof}

\begin{lemma}[Properties of PrivUnit and ScalarDP]\label{lemma:properties-privunit}
    Assume that $\varepsilon_1 \in [0, d]$.
    Then, $z = \privunit(u / \norm{u})$ and $\hat r = \scalardp(\norm{u})$ satisfy
    \begin{align*}
        \norm{z}^2   & = O\ab(\frac{d}{\varepsilon_1} \vee \frac{d}{(e^{\varepsilon_1} - 1)^2}), \\
        \abs{\hat r} & = O\ab(\frac{e^{\varepsilon_2}}{e^{\varepsilon_2} - 1} \cdot C),
    \end{align*}
    with probability 1.
\end{lemma}
\begin{proof}
    The first inequality follows from Proposition 4 in~\citet{bhowmick2018protection}.

    From the definition of $\hat r$, we have $\abs{\hat r} \leq a \abs{\hat J - b} \leq a (k + b)$.
    Substituting, $k = \ceil{e^{\varepsilon_2/3}}, a = \frac{e^{\varepsilon_2} + k}{e^{\varepsilon_2} - 1} \frac{C}{k}$ and $b = \frac{k(k + 1)}{2(e^{\varepsilon_2} + k)}$,
    we obtain the second inequality.
\end{proof}

\begin{lemma}[Tail bounds for PrivUnit]\label{lemma:tail-privunit}
    Let $z_i = \privunit(u_i / \norm{u_i})$ and $\hat r_i = \scalardp(\norm{u_i})$ for $u_i \in \R^d~(\norm{u_i} \leq C)$
    with $\varepsilon_1, \varepsilon_2 = O(1)$.
    Then, for any $v_i \in \R^d$, we have
    \begin{align*}
        \frac{1}{M} \sum_{i=1}^M \langle \hat r_i \cdot z_i - u_i, v_i \rangle         & = O\ab(\sqrt{\frac{C^2d \sum_{i=1}^M \norm{v_i}^2}{M^2}} \cdot q), \\
        \norm{\frac{1}{M} \sum_{i=1}^M (\hat r_i \cdot z_i - u_i)}^2                   & = O\ab(\frac{d C^2 (1 + q^2)}{M}),                                 \\
        \frac{1}{M} \sum_{i=1}^M \hat s_i - \frac{1}{M} \sum_{i=1}^M \norm{\Delta_i}^2 & = O\ab(C^2 \sqrt{\frac{1}{M}} \cdot q),
    \end{align*}
    with probability at least $1 - e^{-q^2/4}$ for any $q \in (0, \sqrt{M}]$.
\end{lemma}
\begin{proof}
    From Lemma~\ref{lemma:properties-privunit} and~\ref{lemma:privunit-unbiased},
    we have $\abs{\langle \hat r_i \cdot z_i - u_i, v_i \rangle} \leq \norm{\hat r_i z_i - u_i}\norm{v_i}= O(\sqrt{d} C \norm{v_i})$
    and $\Expec{\langle \hat r_i \cdot z_i - u_i, v_i \rangle} = 0$.
    Thus, from the Hoeffding inequality, we have
    \begin{align*}
        \frac{1}{M} \sum_{i=1}^M \langle \hat r_i \cdot z_i - u_i, v_i \rangle & = O\ab(\sqrt{\frac{d C^2 \sum_{i=1}^M \norm{v_i}^2}{M^2}} \cdot q),
    \end{align*}
    with probability at least $1 - 2e^{-2q^2}$ for any $q > 0$.

    For the second inequality, Lemma~\ref{lemma:properties-privunit} and~\ref{lemma:privunit-unbiased}
    imply $\norm{\hat r_i \cdot z_i - u_i} = O(\sqrt{d} C)$ and $\Expec{\hat r_i \cdot z_i - u_i} = 0$.
    Thus, using the vector Bernstein inequality in Lemma~\ref{lemma:vec-bernstein}, we have
    \begin{align*}
        \norm{\frac{1}{M}\sum_{i=1}^M (\hat r_i \cdot z_i - u_i)} & = O\ab(\sqrt{\frac{d}{M}}C (1 + q)),
    \end{align*}
    with probability at least $1 - e^{-q^2 / 4}$ for $q \in (0, \sqrt{M})$.
    This yields
    \begin{align*}
        \norm{\frac{1}{M}\sum_{i=1}^M (\hat r_i \cdot z_i - u_i)}^2 & = O\ab(\frac{d C^2 (1 + q^2)}{M}).
    \end{align*}

    For the third inequality, from the definition of $\hat s_i$ and Lemma~\ref{lemma:properties-privunit},
    we have
    \begin{align*}
        \abs{\hat s_i} = \abs{\frac{1}{1 + c_1} (\hat r^2 - c_2 \hat r - c_3)} = O(C^2).
    \end{align*}
    Thus, from the Hoeffding inequality, we have
    \begin{align*}
        \frac{1}{M} \sum_{i=1}^M \hat s_i - \frac{1}{M} \sum_{i=1}^M\norm{\Delta_i}^2 \leq \frac{1}{M} \sum_{i=1}^M \hat s_i - \frac{1}{M} \sum_{i=1}^M \Expec{\hat s_i} = O\ab(C^2 q \sqrt{\frac{1}{M}}),
    \end{align*}
    with probability at least $1 - e^{-q^2 / 2}$ for any $q > 0$.
    For the first inequality, we used Lemma~\ref{lemma:privunit-norm}.
\end{proof}

\section{Proofs for Section~\ref{sec:privacy}} \label{appendix:privacy}
The result for the PrivUnit follows from Lemma~\ref{lemma:privunit-unbiased}.

To tightly audit the privacy leakage of the Gaussian mechanism, we adopt the R\'{e}nyi Differential Privacy (RDP)~\cite{mironov2017renyi}.
\begin{definition}[RDP]
    For any $\alpha \in (1, \infty)$ and any $\varepsilon > 0$, a mechanism $M:\mathcal{X} \to \mathcal{Y}$ is said to be (local) $(\alpha, \varepsilon)-RDP$
    if for any inputs $x, x' \in \mathcal{X}$,
    \begin{align*}
        D_{\alpha}(M(x)\mid M(x')) := \frac{1}{\alpha - 1} \log \Expec[\theta \sim M(x')]{\ab(\frac{M(x)(\theta)}{M(x')(\theta)})^\alpha}\leq \varepsilon.
    \end{align*}
\end{definition}

\paragraph{LDP case}
Since the $l^2$-sensitivity of the local computation at each step is bounded by $2C$,
as shown in~\citet{mironov2017renyi}, Gaussian mechanism is $(\alpha, \alpha \rho)$-RDP, where $\rho = 2 C^2/\sigma^2$

The RDP bound can be converted into the $(\epsilon, \delta)$-DP bound using the following lemma:
\begin{lemma}[\citet{mironov2017renyi}]\label{lemma:dp-conversion}
    Let $M$ be $(\alpha, \varepsilon)$-RDP for $\alpha \in (1, \infty)$.
    Then, $M$ is $(\epsilon + \log (1/\delta) / (\alpha - 1), \delta)$-DP for every $\delta \in (0, 1)$.
\end{lemma}
Applying this lemma, we obtain the result for the Gaussian mechanism.

\paragraph{CDP case}
The $l^2$-sensitivity of $\bar \Delta^{(t)}$ and $\frac{1}{M}\sum_{i=1}^M \norm{\Delta^{(t)}_i}^2$
are bounded by $2C / M$ and $C^2 / M$, respectively.
Thus, $\bar c^{(t)}$ and $\frac{1}{M}\sum_{i=1}^M \norm{\Delta^{(t)}_i}^2 + \xi^{(t)}$
satisfies $(\alpha, 2\alpha C^2/M\sigma^2)$-RDP and $(\alpha, \frac{\alpha C^4}{2M^2 \sigma_\xi^2})$-RDP, respectively.
Then, the entire training process with $T$ iterations satisfy $(\alpha, \alpha (\rho + \rho_\xi))$-RDP, where $\rho = 2C^2T/M\sigma^2, \rho_\xi = C^4T/2M^2\sigma_\xi^2$.
Applying Lemma~\ref{lemma:dp-conversion} yields Proposition~\ref{prop:cdp}.

\section{Proof for Theorem~\ref{thm:non-convex-ldp} and~\ref{thm:non-convex-cdp}} \label{appendix:utility}

To simplify the notation, let
\begin{align*}
    h_i^{(t)}           & := -\Delta_i^{(t)} / (\eta_l \tau) = \frac{1}{\tau}\sum_{k=0}^{\tau - 1} \grad F_i(w^{(t, k)}_i),                                                                         \\
    \bar h^{(t)}        & := -\bar \Delta^{(t)} / (\eta_l \tau) = \frac{1}{M} \sum h_i^{(t)},                                                                                                       \\
    \bar \epsilon^{(t)} & := -(\bar c^{(t)} - \bar \Delta^{(t)}) / (\eta_l \tau)                                                                                                                    \\
    \delta_s^{(t)}      & := \begin{cases}
                                 \frac{1}{M} \sum_{i=1}^M \norm{c_i^{(t)}}^2 - d\sigma^2 - \frac{1}{M} \sum_{i=1}^M \norm{\Delta^{(t)}_i}^2 & \text{for LDP-FedEXP with Gaussian mechanism}, \\
                                 \frac{1}{M} \sum_{i=1}^M \hat s_i^{(t)} - \frac{1}{M} \sum_{i=1}^M \norm{\Delta^{(t)}_i}^2                 & \text{for LDP-FedEXP with PrivUnit},           \\
                                 \xi^{(t)}                                                                                                  & \text{for CDP-FedEXP}.
                             \end{cases}
\end{align*}
Then, the global step size $\eta_g^{(t)}$ is given by
\begin{align}
    \eta_g^{(t)} & = \max\ab\{1, \frac{\frac{1}{M} \sum_{i=1}^M \norm{h^{(t)}_i}^2 + \delta_s^{(t)} / (\eta_l\tau)^2}{\norm{\bar h^{(t)} + \bar \epsilon^{(t)}}^2}\}. \label{eq:eta-g-by-s}
\end{align}

From the smoothness of $F$, $F(w^{(t + 1)})$ satisfies the following:
\begin{align}
    F(w^{(t+1)}) - F(w^{(t)}) & \leq -\eta_g \eta_l \tau \langle \grad F(w^{(t)}), \bar h^{(t)} + \bar \epsilon^{(t)} \rangle + \frac{(\eta_g^{(t)})^2 \eta_l^2 \tau^2 L}{2} \|\bar h^{(t)} + \bar \epsilon^{(t)}\|^2,  \notag             \\
                              & \leq -\eta_g \eta_l \tau\Bigg[ \langle \grad F(w^{(t)}), \bar h^{(t)} + \bar \epsilon^{(t)} \rangle              \notag                                                                                    \\
                              & \quad - \frac{\eta_l \tau L}{2} \max\ab\{\frac{1}{M} \sum_{i=1}^M \norm{h^{(t)}_i}^2 + \delta_s^{(t)} / (\eta_l\tau)^2, \norm{\bar h^{(t)} + \bar \epsilon^{(t)}}^2\}\Bigg]. \label{eq:descent-inequality}
\end{align}
Here, the second inequality follows from Eq.~\eqref{eq:eta-g-by-s}.

For the right-hand side of Eq.~\eqref{eq:descent-inequality}, we have
\begin{align*}
    \langle \grad F(w^{(t)}), \bar h^{(t)} + \bar \epsilon^{(t)} \rangle & = \langle \grad F(w^{(t)}), \bar h^{(t)} \rangle + \langle \grad F(w^{(t)}), \bar \epsilon^{(t)} \rangle                                                                                 \\
                                                                         & = \frac{1}{2} \ab(\norm{\grad F(w^{(t)})}^2 + \norm{\bar h^{(t)}}^2 - \norm{\grad F(w^{(t)}) - \bar h^{(t)}}^2) + \langle \grad F(w^{(t)}), \bar \epsilon^{(t)} \rangle                  \\
                                                                         & \geq \frac{1}{2} \norm{\grad F(w^{(t)})}^2 - \frac{1}{2} \norm{\grad F(w^{(t)}) - \bar h^{(t)}}^2 - \norm{\grad F(w^{(t)})} \norm{\bar \epsilon^{(t)}}                                   \\
                                                                         & \geq \frac{1}{2} \norm{\grad F(w^{(t)})}^2 - \frac{1}{2} \norm{\grad F(w^{(t)}) - \bar h^{(t)}}^2 - \frac{1}{2}\ab(\frac{1}{2}\norm{\grad F(w^{(t)})}^2 + 2\norm{\bar \epsilon^{(t)}}^2) \\
                                                                         & \geq \frac{1}{4} \norm{\grad F(w^{(t)})}^2 - \frac{1}{2M} \sum_{i=1}^M \norm{\grad F_i(w^{(t)}) - h_i^{(t)}}^2 - \norm{\bar \epsilon^{(t)}}^2,                                           \\
    \norm{\bar h^{(t)} + \bar \epsilon^{(t)}}^2                          & \leq 2\norm{\bar h^{(t)}}^2 + 2 \norm{\bar \epsilon^{(t)}}^2,                                                                                                                            \\
                                                                         & \leq \frac{2}{M}\sum_{i=1}^M \norm{h_i^{(t)}}^2 + 2 \norm{\bar \epsilon^{(t)}}^2.
\end{align*}

Substituting the above inequalities into Eq.~\eqref{eq:descent-inequality}, we have
\begin{align}
    F(w^{(t + 1)}) - F(w^{(t)}) & \leq -\eta_g\eta_l\tau \Bigg[\frac{1}{4} \norm{\grad F(w^{(t)})}^2 - \frac{1}{2M} \sum \norm{\grad F_i(w^{(t)}) - h_i^{(t)}}^2 -\norm{\bar \epsilon^{(t)}}^2  \notag                                                                                         \\
                                & \quad - \frac{\eta_l \tau L}{2} \max\ab\{\frac{1}{M} \sum_{i=1}^M \norm{h^{(t)}_i}^2 + \delta_s^{(t)} / (\eta_l\tau)^2, \frac{2}{M}\sum_{i=1}^M \norm{h_i^{(t)}}^2 + 2 \norm{\bar \epsilon^{(t)}}^2\}\Bigg]                                                  \\
                                & \leq -\eta_g\eta_l\tau \Bigg[\frac{1}{4} \norm{\grad F(w^{(t)})}^2 - \frac{1}{2M} \sum \norm{\grad F_i(w^{(t)}) - h_i^{(t)}}^2 - \eta_l \tau L \cdot \underbrace{\frac{1}{M} \sum_{i=1}^M \norm{h_i^{(t)}}^2}_{:=R} \notag                                   \\
                                & \quad -\underbrace{\ab(\norm{\bar \epsilon^{(t)}}^2 + \frac{\eta_l \tau L}{2} \max\ab\{\frac{\delta_s^{(t)}}{(\eta_l\tau)^2} - \frac{1}{M} \sum_{i=1}^M \norm{h^{(t)}_i}, 2 \norm{\bar \epsilon^{(t)}}^2\})}_{:=T_4} \Bigg]. \label{eq:descent-inequality-2}
\end{align}

As in the proof of Theorem 2 in~\citet{jhunjhunwala2023fedexp}, we have
\begin{align*}
    R & \leq \frac{1}{M} \sum \norm{h_i^{(t)}}^2                                                                                                        \\
      & \leq \frac{1}{M} \sum \norm{h_i^{(t)} - \grad f_i(w^{(t)}) + \grad f_i(w^{(t)}) - \grad F(w^{(t)}) + \grad F(w^{(t)})}^2                        \\
      & \leq \frac{3}{M} \sum \ab(\norm{h_i^{(t)} - \grad f_i(w^{(t)})}^2 + \norm{\grad f_i(w^{(t)}) - \grad F(w^{(t)}) }^2+ \norm{\grad F(w^{(t)})}^2) \\
      & \leq \frac{3}{M} \sum_{i=1}^M \norm{h_i^{(t)} - \grad F_i(w^{(t)})}^2 + 3\norm{\grad F(w^{(t)})}^2 + O(\sigma_g^2).
\end{align*}

Substituting $R$ into~Eq.~\eqref{eq:descent-inequality-2}, we arrive at
\begin{align*}
    F(w^{(t + 1)} - F(w^{(t)}))
     & \leq -\eta_g^{(t)}\eta_l \tau \Bigg[\frac{1}{4} \norm{\grad F(w^{(t)})}^2 - \frac{1}{2M} \sum \norm{\grad F_i(w^{(t)}) - h_i^{(t)}}^2 - \eta_l \tau L \cdot R - T_4\Bigg]                      \\
     & \leq -\eta_g^{(t)}\eta_l \tau \Bigg[\frac{1}{4} \norm{\grad F(w^{(t)})}^2 - \frac{1}{2M} \sum \norm{\grad F_i(w^{(t)}) - h_i^{(t)}}^2 - \underbrace{O(\eta_l \tau L \sigma_g^2)}_{:=T_3} - T_4 \\
     & \quad - \eta_l \tau L \cdot \ab(\frac{3}{M} \sum_{i=1}^M \norm{h_i^{(t)} - \grad F_i(w^{(t)})}^2 + 3\norm{\grad F(w^{(t)})}^2)  \Bigg]                                                         \\
     & \leq -\eta_g^{(t)}\eta_l \tau \Bigg[\frac{1}{8} \norm{\grad F(w^{(t)})}^2 - \frac{\eta_l \tau L}{M} \sum_{i=1}^M \norm{\grad F_i(w^{(t)}) - h_i^{(t)}}^2  - T_3 - T_4 \Bigg]                   \\
     & \leq -\eta_g^{(t)}\eta_l \tau \Bigg[\frac{1}{8} \norm{\grad F(w^{(t)})}^2 - \underbrace{O\ab(\eta_l^2\tau^2 L^2 \sigma_g^2)}_{T_2} - T_3 - T_4 \Bigg].
\end{align*}
Here, we used $\eta_l \leq 1/(24\tau L)$ and Lemma 7 in~\citet{jhunjhunwala2023fedexp}.

Averaging over $T$ iterations, we have
\begin{align*}
    \frac{\sum \eta_g^{(t)} \norm{\grad F(w^{(t)})}^2}{\sum \eta_g^{(t)}} & \leq O\ab(\frac{(F(w^{(0)}) - F^*)}{\sum \eta_g^{(t)}\eta_l \tau} + T_2 + T_3 + T_4),
\end{align*}
which implies
\begin{align*}
    \min \norm{\grad F(w^{(t)})}^2 & \quad \leq O\ab(\frac{F(w^{0}) - F^*}{\sum \eta_g^{(t)}\eta_l \tau} + T_2 + T_3 + T_4).
\end{align*}

The remaining task is to evaluate $T_4$.
Recall that $T_4$ is defined as
\begin{align*}
    T_4
     & = \norm{\bar \epsilon^{(t)}}^2 + \frac{\eta_l \tau L}{2} \max\ab\{\frac{\delta_s^{(t)}}{(\eta_l\tau)^2} - \frac{1}{M} \sum_{i=1}^M \norm{h^{(t)}_i}, 2 \norm{\bar \epsilon^{(t)}}^2\} \\
     & \leq  \ab(1 + \eta_l \tau L) \norm{\bar \epsilon^{(t)}}^2 + \frac{L}{\eta_l \tau} \ab(\delta_s^{(t)} - \frac{1}{M} \sum_{i=1}^M \norm{\Delta_i^{(t)}}^2).
\end{align*}

For LDP-FedEXP with Gaussian mechanism,
Lemma~\ref{lemma:gaussian-tail} and~\ref{lemma:chi-squared-tail} yield
\begin{align*}
    \norm{\bar \epsilon^{(t)}}^2                                           & \leq \frac{d}{(\eta_l \tau)^2} \cdot \ab[1 + q^2] \frac{\sigma^2}{M} = O\ab(\frac{q^2}{(\eta_l \tau)^2} \frac{d \sigma^2}{M}), \\
    \frac{1}{M} \sum_{i=1}^M \norm{\varepsilon^{(t)}_i}^2                  & = d \cdot \ab[1 + \frac{q^2}{\sqrt{Md}}] \sigma^2                                                                              \\
    \frac{1}{M} \sum_{i=1}^M \lrangle{\Delta_i^{(t)}, \varepsilon_i^{(t)}} & \leq q \cdot \ab(\frac{\sigma}{M} \sqrt{\sum_{i=1}^M \norm{\Delta_i^{(t)}}^2})                                                 \\
                                                                           & \leq \frac{1}{2M} \sum_{i=1}^M \norm{\Delta_i^{(t)}}^2 + \frac{q^2 \sigma^2}{2M},
\end{align*}
with probability $1 - T e^{-c \cdot q^2}$ for $q \in [1, \sqrt{M}]$, where $c$ is a numerical constant.
Here, we used the union bound over $t = 1, \dots, T$.
Then, we obtain
\begin{align*}
    \delta_s^{(t)} - \frac{1}{M} \sum_{i=1}^M \norm{\Delta^{(t)}_i}
     & = \frac{1}{M} \sum_{i=1}^M \norm{c_i^{(t)}}^2 - d\sigma^2 - \frac{2}{M} \sum_{i=1}^M \norm{\Delta^{(t)}_i}^2                                                                                    \\
     & = \frac{1}{M} \sum_{i=1}^M \norm{\Delta_i^{(t)} + \varepsilon_i^{(t)}}^2 - d\sigma^2 - \frac{2}{M} \sum_{i=1}^M \norm{\Delta^{(t)}_i}^2                                                         \\
     & = \frac{1}{M} \sum_{i=1}^M \norm{\varepsilon_i^{(t)}}^2 - d\sigma^2 + \frac{2}{M} \sum_{i=1}^M \lrangle{\Delta_i^{(t)}, \varepsilon_i^{(t)}} - \frac{1}{M} \sum_{i=1}^M \norm{\Delta^{(t)}_i}^2 \\
     & = q^2 \cdot \sqrt{\frac{d}{M}} \sigma^2 + \frac{q^2 \sigma^2}{M}.
\end{align*}
Substituting these concentration inequalities, we obtain
\begin{align*}
    T_4 & = O\ab(\ab(1 + \eta_l \tau L) \frac{q^2}{(\eta_l \tau)^2} \frac{d\sigma^2}{M} + \frac{L}{\eta_l \tau} \ab(q^2 \cdot \sqrt{\frac{d}{M}} \sigma^2 + \frac{q^2 \sigma^2}{M})) \\
        & = O\ab(\frac{L\sigma^2q^2}{\eta_l \tau}\ab[\frac{d}{M} + \sqrt{\frac{d}{M}}]),
\end{align*}
since $q \geq 1$ and $\eta_l = \Theta(1/L\tau)$.

For LDP-FedEXP with PrivUnit, Lemma~\ref{lemma:tail-privunit} yields
\begin{align*}
    \delta_s^{(t)} = \frac{1}{M} \sum_{i=1}^M \hat s_i^{(t)} = O(C^2q \sqrt{\frac{1}{M}}), \\
    \norm{\bar \epsilon^{(t)}}^2 = O\ab(\frac{dC^2(1 + q^2)}{M(\eta_l \tau)^2})
\end{align*}
with probability $1 - T e^{-c \cdot q^2}$ for $q \in [1, \sqrt{M}]$, where $c$ is a numerical constant.
Substituting these concentration inequalities, we obtain
\begin{align*}
    T_4 & = O\ab(\ab(1 + \eta_l \tau L) \frac{dC^2(1 + q^2)}{M(\eta_l \tau)^2}) + O\ab(\frac{L}{\eta_l\tau} C^2 q \sqrt{\frac{1}{M}}) \\
        & = O\ab(\frac{L C^2 q^2}{\eta_l \tau} \ab[\frac{d}{M} + \sqrt{\frac{1}{M}}])                                                 \\
        & = O\ab(\frac{L \sigma^2 q^2}{\eta_l \tau} \ab[\frac{d}{M} + \sqrt{\frac{1}{M}}]).
\end{align*}

For CDP-FedEXP, we have
\begin{align*}
    \delta_s^{(t)} = \xi^{(t)}_i = O(q \sigma_\xi), \\
    \norm{\bar \epsilon^{(t)}}^2 = O\ab(\frac{q}{(\eta_l \tau)^2}\frac{d\sigma^2}{M}),
\end{align*}
with probability $1 - T e^{-c \cdot q^2}$ for $q \in [1, \sqrt{M}]$, where $c$ is a numerical constant.
Substituting these concentration inequalities, we obtain
\begin{align*}
    T_4 & = O\ab(\ab(1 + \eta_l \tau L) \frac{q}{(\eta_l \tau)^2} \frac{d\sigma^2}{M} + \frac{L}{\eta_l \tau}q \sigma_\xi) \\
        & = O\ab(\frac{L\sigma^2q^2}{\eta_l \tau}\frac{d}{M}).
\end{align*}

\section{Supplementary Material for Numerical Experiments}\label{appendix:experiments}
Here, we provide additional details and results for the numerical experiments in Section~\ref{sec:experiments}.

\subsection{Detailed Setup}
\paragraph{Hyperparameter Tuning}
We tuned the hyper parameters (local learning rate $\eta_l$ and clipping threshold $C$)
via grid search and select the best hyperparameters which maximize the test accuracy for the realistic dataset or minimize the training loss for the synthetic dataset averaged over the last 5 rounds.
In the synthetic experiment, the grid for $\eta_l$ is $\{0.01, 0.03, 0.1, 0.3, 1\}$ and for $C$ is $\{0.1, 0.3, 1, 3, 10\}$.
In the realistic experiment, the grid for $\eta_l$ is $\{0.0001, 0.0003, 0.001, 0.003, 0.01\}$ and for $C$ is $\{0.1, 0.3, 1, 3, 10\}$.
We summarize the best performing hyperparameters in Table~\ref{table:hyperparameters}.
\begin{table}[h]
    \caption{Best hyperparameters selected via grid search for DP-FedEXP, DP-FedAvg, and DP-SCAFFOLD.}
    \centering
    \begin{tabular}{ccccccccccc}
        \hline
                  &                & \multicolumn{2}{c}{FedEXP} & \multicolumn{2}{c}{FedAvg} & \multicolumn{2}{c}{SCAFFOLD}                        \\
        Dataset   & DP type        & $\eta_l$                   & $C$                        & $\eta_l$                     & $C$ & $\eta_l$ & $C$ \\
        \hline
        Synthetic & LDP (Gaussian) & 0.003                      & 0.3                        & 0.003                        & 3   & 0.003    & 0.3 \\
                  & LDP (PrivUnit) & 0.003                      & 1                          & 0.003                        & 3   & 0.003    & 0.3 \\
                  & CDP            & 0.001                      & 0.3                        & 0.003                        & 3   & 0.001    & 1   \\
        \hline
        MNIST     & LDP (Gaussian) & 0.03                       & 0.1                        & 0.03                         & 0.3 & 0.1      & 0.1 \\
                  & LDP (PrivUnit) & 0.03                       & 0.3                        & 0.03                         & 0.3 & 0.03     & 0.1 \\
                  & CDP            & 0.1                        & 0.3                        & 0.1                          & 1   & 0.1      & 0.3 \\
        \hline
    \end{tabular}
    \label{table:hyperparameters}
\end{table}

\paragraph{Synthetic Dataset}
In principle, we follow a similar procedure in~\citet{li2020federated,jhunjhunwala2023fedexp}.
First, we generate the true model $w^*$ by sampling from the standard normal distribution.
Then, we generate vectors $x_i \in \R^d$ according to $x_i \sim \mathcal{N}(m_i, I_d)$, where $m_i \sim \mathcal{N}(u_i, 1), u_i\sim \mathcal{N}(0, 0.1)$.
The client objective is defined as $f_i(w) := \norm{x_i^\top w - y_i}^2$, where $y_i = x_i^\top w^*$.

\clearpage
\paragraph{Model Architectures}
We summarize the architectures of the models used in the MNIST experiments in Table~\ref{table:models}.

\begin{table}[h]
    \caption{Model architectures used in the experiments.}
    \centering
    \begin{tabular}{|c|c|}
        \hline
        Setting & Model Architecture         \\
        \hline
        CDP     &
        \begin{tabular}{c}
            Convolutional layer (4 filters, 4x4) \\
            Convolutional layer (8 filters, 4x4) \\
            Fully connected layer (128 $\to$ 32) \\
            ReLU activation                      \\
            Fully connected layer (32 $\to$ 10)  \\
            Softmax activation
        \end{tabular} \\
        \hline
        LDP     &
        \begin{tabular}{c}
            Convolutional layer (2 filters, 4x4) \\
            Convolutional layer (1 filters, 4x4) \\
            Fully connected layer (16 $\to$ 10)  \\
            Softmax activation
        \end{tabular} \\
        \hline
    \end{tabular}
    \label{table:models}
\end{table}

\subsection{Additional Results}
Here, we provide additional results omitted in the main text due to space constraints.

\paragraph{Adaptivity in Global Step Size}
Fig.~\ref{fig:eta_g_hist} plots the global step size $\eta_g^{(t)}$ of each algorithm.
Interestingly, in the synthetic experiment, the global step size of DP-FedEXP decreases as the training progresses.
This enables to speed up the training process and to mitigate the effect of the DP noise on the converged model at the same time.
This phenomenon clearly demonstrates the advantage of the adaptive step size in DP-FL.
\begin{figure}[h]
    \centering
    \includegraphics[width=\textwidth]{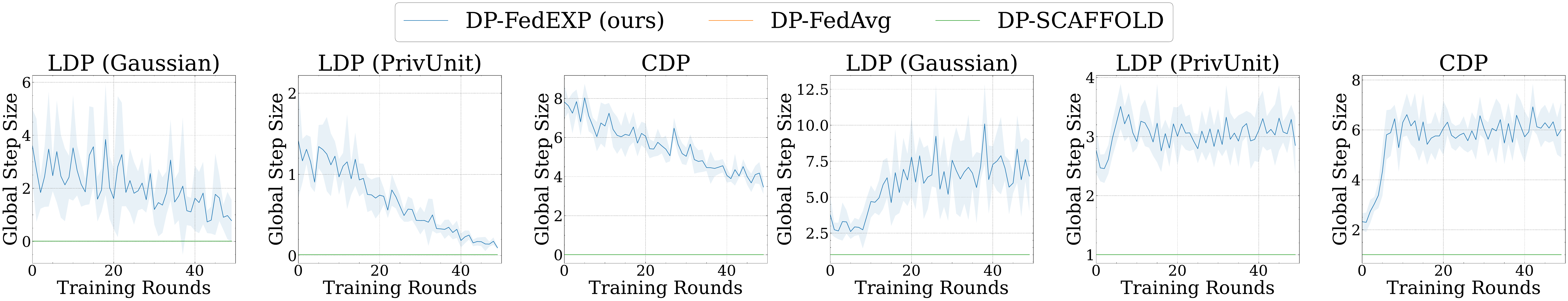}
    \caption{Global step sizes for the synthetic dataset (left) and the MNIST dataset (right).}
    \label{fig:eta_g_hist}
\end{figure}

\paragraph{Additional Results for the MNIST Dataset}
To evaluate the performance of the model at the end of the training process,
we report the test accuracy averaged over the last 5 rounds in Table~\ref{table:test-accuracy}.
Our proposed DP-FedEXP comprehensively outperforms the baselines in all settings.
\begin{table}[h]
    \caption{Test accuracy of algorithms on the MNIST dataset averaged over the last 5 rounds.
        Mean (standard deviation) over 5 runs with different random seeds is reported.}
    \centering
    \begin{tabular}{cccccc}
        \hline
        DP Type        & DP-FedEXP             & DP-FedAvg     & DP-SCAFFOLD  \\
        \hline
        LDP (Gaussian) & \textbf{80.24} (0.94) & 78.69 (1.26)  & 66.89 (2.29) \\
        \hline
        LDP (PrivUnit) & \textbf{79.65} (1.23) & 78.40  (1.18) & 56.83 (3.95) \\
        \hline
        CDP            & \textbf{94.57} (0.19) & 92.88 (0.29)  & 86.61 (0.52) \\
        \hline
    \end{tabular}
    \label{table:test-accuracy}
\end{table}


\end{document}